\documentclass{article}
\usepackage[nonatbib,final]{neurips_2021}
\usepackage[utf8]{inputenc} %
\usepackage[T1]{fontenc}    %
\usepackage[colorlinks = true,
            linkcolor = royalazure,
            urlcolor  = royalazure,
            citecolor = royalazure,
            anchorcolor = royalazure]{hyperref}
\usepackage{booktabs}       %
\usepackage{amsfonts}       %
\usepackage{nicefrac}       %
\usepackage{microtype}      %
\usepackage{wrapfig}
\usepackage{caption}
\usepackage{csquotes}
\usepackage{seqsplit}
\usepackage[numbers,compress]{natbib}
\usepackage{url}            %

\usepackage{floatrow} %
\newfloatcommand{capbtabbox}{table}[][\FBwidth]

\newcommand{\myparagraph}[1]{\par\noindent\textbf{{#1}.}} %

\newcounter{arxiv}
\usepackage{adjustbox}
\usepackage[noend]{algpseudocode}
\usepackage{amsfonts}
\usepackage{amsmath}
\usepackage{amssymb}
\usepackage{amsthm}
\usepackage{bm}
\usepackage{bbm}
\usepackage{booktabs}
\usepackage[T1]{fontenc}
\usepackage{graphicx}
\usepackage{latexsym}
\usepackage{mathtools}
\usepackage{multirow}
\usepackage{paralist}
\usepackage[table]{xcolor}  %
\usepackage{xspace}
\usepackage{listings}
\usepackage{minitoc}

\newcommand \Xcal {\mathcal X}

\newcommand \Fcal {\mathcal F}

\newcommand \Ccal {\mathcal C}

\usepackage{bm}
\newcommand \xv {\bm{x}}

\newcommand \e {e}

\newtheorem{theorem}{Theorem}
\newtheorem{lemma}[theorem]{Lemma}

\newtheorem{proposition}[theorem]{Proposition}

\theoremstyle{definition}

\newcommand{\kl}{{\mathrm{KL}}}

\newcommand \eps \varepsilon
\renewcommand \epsilon \varepsilon

\newcommand \argmin {\operatorname*{arg\,min}} %

\newcommand \reals {\mathbb{R}}

\newcommand \expect {\mathbb{E}}

\usepackage{tikz}
\usetikzlibrary{shapes}
\usetikzlibrary{positioning}
\usetikzlibrary{plotmarks}
\usetikzlibrary{patterns}
\usetikzlibrary{intersections,shapes.arrows}
\usepackage{pgfplots}
\usetikzlibrary{pgfplots.fillbetween}
\definecolor{cadmiumgreen}{rgb}{0.0, 0.42, 0.24}
\definecolor{oldmauve}{rgb}{0.4, 0.19, 0.28}
\definecolor{royalazure}{rgb}{0.0, 0.22, 0.66}
\definecolor{harvardcrimson}{rgb}{0.79, 0.0, 0.09}
\definecolor{lightmauve}{rgb}{0.86, 0.82, 1.0}
\definecolor{darkbrown}{rgb}{0.4, 0.26, 0.13}

\usepackage{enumitem}
\usepackage{tcolorbox}

\definecolor{azure}{rgb}{0.0, 0.5, 1.0}
\tcbuselibrary{skins}
\tcbset{enhanced}

\makeatletter
\newcommand\footnoteref[1]{\protected@xdef\@thefnmark{\ref{#1}}\@footnotemark}
\makeatother

\newcommand{\name}{{\fontfamily{bch}\selectfont{\textsc{Mauve}}}\xspace}
\newcommand{\tabemph}[1]{\cellcolor{lightmauve!30}\textcolor{black!50!royalazure}{#1}}%

\definecolor{Code}{rgb}{0,0,0} 
\definecolor{Decorators}{rgb}{0.5,0.5,0.5} 
\definecolor{Numbers}{rgb}{0.5,0,0} 
\definecolor{MatchingBrackets}{rgb}{0.25,0.5,0.5} 
\definecolor{Keywords}{rgb}{0,0,1} 
\definecolor{self}{rgb}{0,0,0} 
\definecolor{Strings}{rgb}{0,0.63,0} 
\definecolor{Comments}{rgb}{0.63,0,0} 
\definecolor{Backquotes}{rgb}{0,0,0} 
\definecolor{Classname}{rgb}{0,0,0} 
\definecolor{FunctionName}{rgb}{0,0,0} 
\definecolor{Operators}{rgb}{0,0,0} 
\definecolor{Background}{rgb}{0.99,0.99,0.99} 

\usepackage{setspace}
\usepackage[ruled, vlined]{algorithm2e}
\SetKwInput{KwInput}{Input} 

\SetCommentSty{mycommfont}

 \title{\name: Measuring the Gap \\ Between Neural Text and Human Text \\ using Divergence Frontiers}

\author{
Krishna Pillutla$^{1}$\quad 
Swabha Swayamdipta$^{2}$\quad 
Rowan Zellers$^{1}$\quad
John Thickstun$^{3}$\quad \\
\textbf{Sean Welleck}$^{1,2}$ \quad
\textbf{Yejin Choi}$^{1,2}$\quad
\textbf{Zaid Harchaoui}$^{4}$
\vspace{0.3cm}\\
  $^{1}$Paul G.\ Allen School of Computer Science \& Engineering, University of Washington \\
  $^{2}$Allen Institute for Artificial Intelligence\\
  $^{3}$Department of Computer Science, Stanford University \\
  $^{4}$Department of Statistics, University of Washington
}

\begin{document}
\maketitle
\doparttoc %
\faketableofcontents %

\begin{abstract}
As major progress is made in open-ended text generation, measuring how close machine-generated text is to human language remains a critical open problem.
We introduce \name, a comparison measure for open-ended text generation, which directly compares the learnt distribution from a text generation model to the distribution of human-written text using divergence frontiers. 
\name scales up to modern text generation models by computing information divergences in a quantized embedding space.
Through an extensive empirical study on three open-ended generation tasks, we find that \name identifies known properties of generated text, scales naturally with model size, and correlates with human judgments, with fewer restrictions than existing distributional evaluation metrics.
\end{abstract} %
\section{Introduction}
\label{sec:intro}
Recent large-scale text generation models show an ability to produce human-like text of remarkable quality and coherence in open-ended generation  \cite{radford2019language,zellers2019grover,brown2020language}.
In this setting, a text generation model forms a distribution over natural language sequences, induced by an autoregressive neural sequence model (e.g., GPT-3 \citep{brown2020language}) paired with a decoding algorithm (e.g., nucleus sampling \citep{holtzman2019curious}).
Generating text amounts to sampling from this distribution, with the goal of obtaining samples that resemble those from the ``true'' distribution of human-written text.

To evaluate how close a generation model's distribution is to that of human-written text, we must consider two types of errors: (I) where the model assigns high probability to sequences which do {\em not} resemble human-written text, and, (II) where the model distribution does not cover the human distribution, i.e., it fails to yield diverse samples.
However, quantifying these aspects in a principled yet computationally tractable manner is challenging, as the text distributions are high-dimensional and discrete, accessed only through samples or expensive model evaluations
\citep{holtzman2019curious,welleck2020consistency,zhang2020trading}.

We develop \name, 
a comparison measure for open-ended text generation.
The proposed measure is efficient, interpretable, and practical for evaluating modern text generation models. 
It captures both types of errors (Figure~\ref{fig:main:illustration})
by building upon \textit{information divergence frontiers} \citep{sajjadi2018assessing,kynknniemi2019improved,djolonga2020precision}, so far underexplored in natural language processing.
The key idea for making the proposed measure computationally tractable, yet effective, is to reduce its measurement to computing Kullback-Leibler divergences in a quantized, low-dimensional space after embedding samples from each distribution with an external language model.
From an end-user's perspective, \name has a simple interface:
given neural text and human text, it yields a scalar measure of the gap between them.

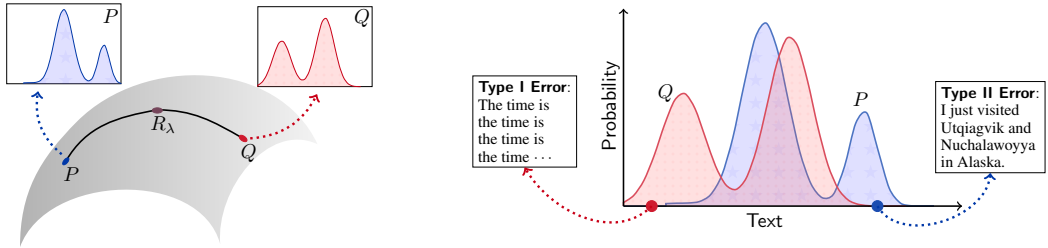
\begin{figure}[t]
    \centering
    \adjustbox{max width=0.37\textwidth}{%
\begin{tikzpicture}

\shade[left color=gray!60,right color=gray!10] 
  (0,0.5) to[out=70,in=105] (3.5, 0.55) to[out=95,in=155] (5.2, 2.3) to[out=115,in=45] (1, 2.7) to[out=-135,in=90] (0, 0.5);
\draw[thick] (0.8, 2) 
    to[out=60 , in=190] (2.4, 2.9)
    to[out=0,in=150] (3.9, 2.4);

\draw[color=royalazure, fill,rotate around={50:(0.8,2)}] (0.8,2) ellipse (2.1pt and 1pt) node[label={[shift={(0.1,0.15)},text=black]270:{$P$}}] (p_manif) {};
\draw[color=harvardcrimson!90, fill,rotate around={-35:(3.9, 2.4)}] (3.9, 2.4) ellipse (2.2pt and 1.2pt) node[label={[shift={(0.1,0.15)}, text=black]270:$Q$}] (q_manif) {};

\draw[color=oldmauve!90, fill] (2.4, 2.89) ellipse (2.5pt and 1.3pt) node[label={[shift={(0.1,0.15)},text=black]272.5:$R_\lambda$}] (r_manif) {};

\node[above = 1.1cm of p_manif](pbox) {
    \begin{tikzpicture}[scale=2]
    \draw (0.001, 0.001) rectangle (1,0.7);
    \def\normalmixP{\x, {
    0.01 * exp( -((\x-0.19)^2)/ (2 * 0.08 * 0.08) ) + 
    0.65 * exp( -((\x-0.5)^2)/ (2 * 0.08 * 0.08) ) + 
    0.34 * exp( -((\x-0.85)^2)/ (2 * 0.05 * 0.05) )
    }}
    \draw[color=royalazure, preaction={fill=blue!60, fill opacity=0.2}, fill opacity=0.2, pattern=fivepointed stars, pattern color=blue, domain=0.15:1,smooth] (0.15, 0) -- plot (\normalmixP) -- (1, 0);
    \node at (0.91, 0.47) {$P$};
    \end{tikzpicture}
};

\node[above right= 0.65cm and 0cm of q_manif](qbox) {
\begin{tikzpicture}[scale=2]
    \draw (0.001, 0.001) rectangle (1,0.7);
    \def\normalmixQ{\x, {
    0.4 * exp( -((\x-0.21)^2)/ (2 * 0.08 * 0.08) ) + 
    0.6 * exp( -((\x-0.59)^2)/ (2 * 0.08 * 0.08) )
    }}
    \draw[color=harvardcrimson, preaction={fill=red!60, fill opacity=0.2}, fill opacity=0.2, pattern=dots, pattern color=red, domain=0:0.9,smooth] (0, 0) -- plot (\normalmixQ) -- (0.9, 0);
    \node at (0.78, 0.46) {$Q$};
    \end{tikzpicture}
};

\draw[->, bend left, dotted, very thick, color=royalazure] (p_manif) to (pbox) ;
\draw[->, bend right, dotted, very thick, color=harvardcrimson] (q_manif) to (qbox.270) ;
\end{tikzpicture} %
}
    \hfill
    \adjustbox{max width=0.55\textwidth}{%
\begin{tikzpicture}[scale=5]
\def\normalmixQ{\x, {
0.4 * exp( -((\x-0.21)^2)/ (2 * 0.08 * 0.08) ) + 
0.6 * exp( -((\x-0.59)^2)/ (2 * 0.08 * 0.08) )
}}
\def\normalmixP{\x, {
0.01 * exp( -((\x-0.19)^2)/ (2 * 0.08 * 0.08) ) + 
0.65 * exp( -((\x-0.5)^2)/ (2 * 0.08 * 0.08) ) + 
0.34 * exp( -((\x-0.85)^2)/ (2 * 0.05 * 0.05) )
}}

\draw[thick,color=royalazure!70, preaction={fill=blue!60, fill opacity=0.2}, fill opacity=0.2, pattern=fivepointed stars, pattern color=blue, domain=0.15:1.1,smooth] (0, 0) -- (0.15, 0) -- plot (\normalmixP) -- (1, 0);
\draw[thick,color=harvardcrimson!70, preaction={fill=red!60, fill opacity=0.2}, fill opacity=0.2, pattern=dots, pattern color=red, domain=0:0.87,smooth] (0, 0) -- plot (\normalmixQ) -- (0.87, 0) -- (1, 0);

\draw[thick,->] (0,0) -- (1.2,0) node[below, xshift=-3.5cm] {\fontfamily{cmss}\selectfont Text};
\draw[thick,->] (0,0) -- (0,0.7) node[left, rotate=90, xshift=-0.8cm, yshift=0.23cm] {\fontfamily{cmss}\selectfont Probability};

\node(plabel) at (0.84, 0.38) {$P$};
\node(qlabel) at (0.15, 0.4) {$Q$};

\node[draw, align=left, xshift=-2.5cm, yshift=-0.5cm, at=(qlabel), scale=0.8](q_ex_text) {
    \fontfamily{cmss}\selectfont
    \textbf{Type I Error}: \\ 
    The time is \\ the time is \\ the time is \\ the time $\cdots$
};

\node[draw, align=left, xshift=2.3cm, yshift=-0.5cm, at=(plabel), scale=0.8](p_ex_text) {
    \fontfamily{cmss}\selectfont
    \textbf{Type II Error}: \\ 
    I just visited  \\
    Utqiagvik and \\
    Nuchalawoyya \\
    in Alaska.
};

\draw[color=harvardcrimson!90, fill] (0.1, 0) circle (0.55pt) {};
\draw[color=royalazure!90, fill] (0.9, 0) circle (0.55pt)  {};

\draw[->, dotted, very thick, color=harvardcrimson] (0.1, 0) to[out=225, in=300] (q_ex_text.270) {};

\draw[->, dotted, very thick, color=royalazure] (0.9, 0) to[out=300, in=265] (p_ex_text.270) {};

\end{tikzpicture}
}
    \caption{
    \small
    \textbf{Left}:
    \name compares the machine text distribution $Q$ to 
    that of human text $P$ by using the family of mixtures $R_\lambda = \lambda P + (1-\lambda) Q$ for $\lambda \in (0, 1)$. 
    \textbf{Right}: Illustration of \textit{Type I errors}, where $Q$ produces degenerate, repetitive text which is unlikely under $P$, and, \textit{Type II errors}, where $Q$ cannot produce plausible human text due to truncation heuristics~\cite{holtzman2019curious}.
    \name measures these errors softly, by using the mixture distribution $R_\lambda$. Varying $\lambda$ in $(0, 1)$ gives a divergence curve and captures a spectrum of soft Type I and Type II errors. 
    \name summarizes the entire divergence curve in a single scalar as the area under this curve.
    }
    \label{fig:main:illustration}
\end{figure}

We summarize our contributions.
First, we introduce \name, a comparison measure between neural text and human text. 
Second, we empirically show that \name is able to
        quantify known properties of generated text with respect to text length, model size, and decoding
        more correctly and with fewer restrictions than existing distributional evaluation metrics.
Third, we find through a human evaluation that \name better correlates with human quality judgements of text. 
Finally, we find that \name can be highly robust to the choice of quantization, embeddings, and scaling.
We open-source a pip-installable Python package to compute \name.\footnote{
        Available from \url{https://github.com/krishnap25/mauve}.
        See Appendix~\ref{supp:software} for an example of the \texttt{mauve} package in action.
    }
\section{\name}
\label{sec:mauve}

We begin by discussing the basics of open-ended text generation, and then introduce \name for measuring the divergence between machine generated text and human text.
\myparagraph{Open-ended Text Generation}
A language model is an estimate $\hat P(\xv)$ of the probability distribution over sequences of text $\xv=(x_1,\ldots,x_{|\xv|})$, consisting of tokens $x_t$ belonging to a fixed vocabulary (e.g. characters, or words).
Prevailing neural autoregressive language models estimate the joint distribution $\hat P(\xv)$ by modeling the conditional distribution $\hat P(x_{t+1}|\xv_{1:t})$ over the next token in a sequence.
The open-ended text generation task asks us to output text $\hat{\xv}_{t+1:|\xv|}$ in continuation of a given context $\xv_{1:t}$.
Unlike targeted generation tasks like translation or summarization, there is no ``correct'' output; the main criteria for open-ended text generation are coherence, creativity, and fluency.

Given a neural autoregressive language model $\hat P$, we can generate open-ended text in a serial, left-to-right fashion, by sampling $\hat{x}_{t+1} \sim \hat P(\cdot|\xv_{1:t})$, $\hat{x}_{t+2} \sim \hat P(\cdot|\xv_{1:t}, \hat{x}_{t+1})$, etc. In practice, this simple decoding algorithm is often modified by adjusting the conditional distribution $\hat P(\cdot|\xv_{1:t})$ to promote more conservative outputs.
The decoding algorithm and the language model taken together define a distribution $Q$ over text, which we call the \textit{model distribution}.
Common decoding algorithms include temperature rescaling~\cite{ackley1985learning}
and truncation~\cite{fan2018heirarchical,holtzman2019curious}. Note that truncation methods in particular create sparsity in $Q$, which leads to degeneracy of some measures including test-set perplexity.

\myparagraph{Sources of Error in Text Generation} 
Our goal in this work is to measure the gap between the model distribution $Q$ and the target distribution $P$ of human text. 
As highlighted in Figure~\ref{fig:main:illustration},
this gap arises from two sources of error: 
\begin{itemize}[itemsep=0cm,leftmargin=1.7cm,topsep=0cm]
\item[(Type I)] $Q$ places high mass on text which is unlikely under $P$, 
\item[(Type II)] $Q$ cannot generate text which is plausible under $P$.
\end{itemize}

The Type I errors are false positives, 
including the common failure case where a model generates text with semantic repetitions \cite{dinan2019second,holtzman2019curious,welleck2020neural} that are highly unlikely to be written by humans.%
\footnote{
Let text $\xv$ with $P(\xv) \gg 0$ be the positive class
and $P(\xv) \approx 0$ be the negative class.
If $Q(\xv) \gg 0$ for some negative $\xv$, 
then the model incorrectly considers it a positive,
so it is a {\em false} positive. 
}
The Type II errors are false negatives, 
which can occur, for instance, because some pieces of plausible human text cannot be generated by truncation-based decoding algorithms such as nucleus sampling~\cite{holtzman2019curious}. 
The gap between $P$ and $Q$ is small only if both of these errors are small.

\myparagraph{Quantifying the Errors}
We formalize the Type I and II errors with the Kullback-Leibler (KL) divergences $\kl(Q|P)$ and $\kl(P|Q)$, respectively. The divergence $\kl(Q|P)$ penalizes $Q$ 
if there exists text $\xv$ such that $Q(\xv)$ is large but $P(\xv)$ is small, so it quantifies the Type I error. 
Likewise, $\kl(P|Q)$ quantifies the Type II error. 

Unfortunately, one or both of the KL divergences $\kl(P|Q)$ and $\kl(Q|P)$ are infinite if the supports of $P$ and $Q$ are not identical, which is often the case in open-ended generation. 
This makes the KL divergence itself unsuitable as an evaluation metric.
We overcome this issue by {\em softly} measuring the two errors using the
mixture distribution
$R_\lambda = \lambda P + (1-\lambda)Q$ for some $\lambda \in (0, 1)$.
In particular, we define the (soft) Type I error at level $\lambda$ as $\kl(Q|R_\lambda)$ and the (soft) Type II error as $\kl(P|R_\lambda)$.

\begin{figure*}[t]
    \centering
    \includegraphics[width=0.98\linewidth]{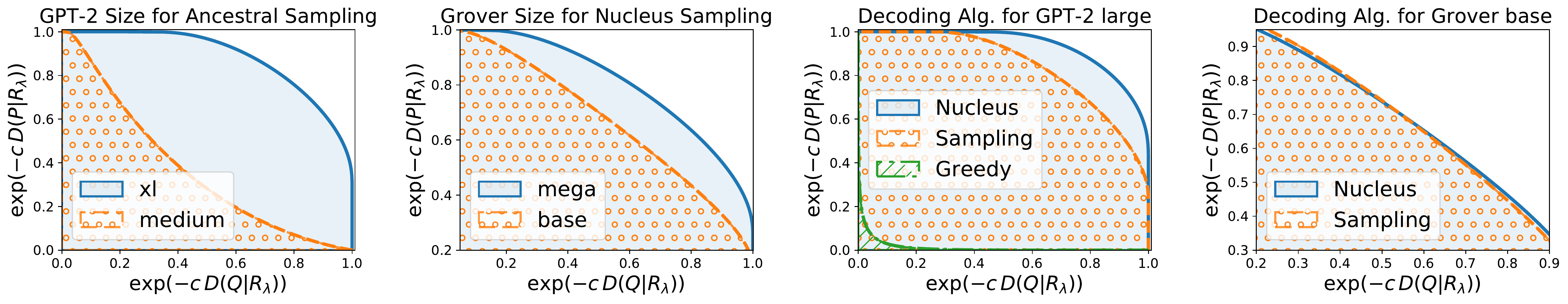}
    \caption{ \small
    Divergence curves for different models (GPT-2~\cite{radford2019language}, Grover~\cite{zellers2019grover}) 
    and decoding algorithms (greedy decoding, ancestral and nucleus sampling). 
    \name is computed as the area of the shaded region, and larger values of \name indicate that $Q$ is closer to $P$.
    In general, \name indicates that generations from larger models and nucleus sampling are closer to human text.
    \textbf{Rightmost}: Nucleus sampling has a slightly smaller Type I error
    than ancestral sampling but a higher Type II error, 
    indicating that ancestral sampling with Grover base produces more degenerate text while nucleus sampling 
    does not effectively cover the human text distribution.
    }
    \label{fig:mauve:illustration}
\end{figure*}

\myparagraph{Summarizing the Errors with a Divergence Curve}
Since the mixture weight $\lambda$ was arbitrary, 
we consider a family of Type I and II error values by varying $\lambda$ between 0 and 1,
in the same spirit as information divergence frontiers~\citep{sajjadi2018assessing,djolonga2020precision}.
This yields a {\em divergence curve},
\begin{align} \label{eq:div-curve}
    \Ccal(P, Q) =
    \Big\{ \big(\exp(-c\, \kl(Q|R_\lambda)), \exp(-c\, \kl(P|R_\lambda)) \big)
    \, :\, 
    R_\lambda = \lambda P + (1-\lambda)Q, \,
    \lambda \in (0, 1) \Big\} \,,
\end{align}
where $c > 0$ is a hyperparameter 
for scaling.
The divergence curve formalizes and encodes information about the trade-off between Type I and II errors.\footnote{More generally, 
the divergence curve $\Ccal(P, Q)$
encodes the \textbf{Pareto frontier}
of $\big(\kl(P|R), \kl(Q|R)\big)$
for all distributions $R$,
not just mixtures of the form $R_\lambda$. 
We prove this in Appendix~\ref{supp:pr}.
}
Figure~\ref{fig:mauve:illustration} illustrates the divergence curves for different models and decoding algorithms.

Our proposed measure, 
$\textbf{\name}(P,Q)$, is the area under the divergence curve $\Ccal(P, Q)$. 
It provides a scalar summary of the trade-off between Type I and Type II errors.
$\name(P,Q)$ lies in $(0,1]$, with a larger value 
meaning that $Q$ is closer to $P$. Further, 
$\name(P, Q) = 1$ if and only if $Q=P$. %
The area under the curve is a common summary of trade-off curves in machine learning~\cite{cortes2005confidence,clemencon2009precision, clemenccon2010overlaying, flach2012machine}.

\myparagraph{Connections to Common Divergences}
The divergence curve encodes more information
than the KL divergence $\kl(P|Q)$, which can be obtained from the second coordinate of the curve $\Ccal(P, Q)$ as $\lambda \to 0$, and the reverse KL divergence $\kl(Q|P)$ which can be obtained from the first coordinate of the curve $\Ccal(P, Q)$ as $\lambda \to 1$. 
Further, the Jensen-Shannon (JS) divergence $\mathrm{JS}(P, Q) = \big(\kl(P|R_{1/2}) + \kl(Q|R_{1/2})\big)/2$, 
can be obtained from the two coordinates of $\Ccal(P, Q)$ at $\lambda=1/2$.
\name summarizes {\em all} of the divergence curve $\Ccal(P, Q)$.

\begin{figure*}[t]
    \centering
    \includegraphics[width=0.8\linewidth]{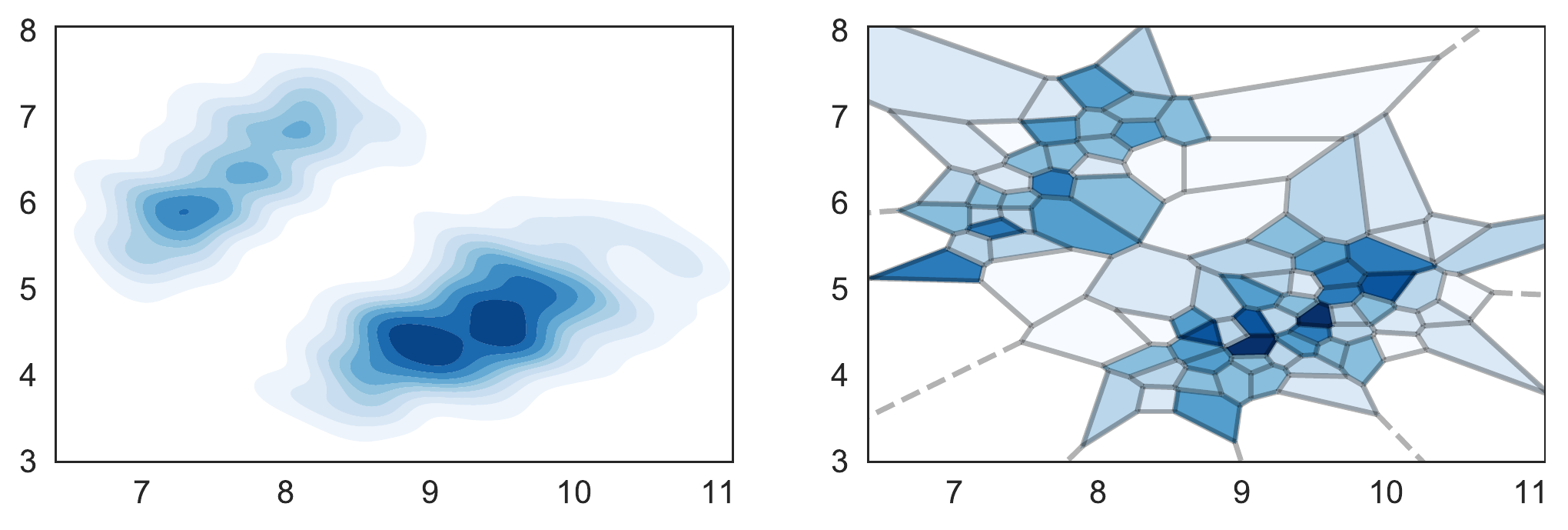}
    \caption{\small
    Illustration of the quantization.
    \textbf{Left}: A continuous two-dimensional distribution $P$.
    \textbf{Right}: A partitioning of the Euclidean plane $\reals^2$ and the corresponding quantized distribution $\tilde P$.
    }
    \label{fig:mauve:quant}
\end{figure*}

\myparagraph{Computing \name for Open-Ended Text Generation}
Each point on the divergence curve $\Ccal(P, Q)$ consists of a coordinate
\begin{align}\label{eqn:klpr}
    \kl(P|R_\lambda) &= \sum_{\xv}P(\xv)\log\frac{P(\xv)}{R_\lambda(\xv)}, 
\end{align}
and a similarly defined coordinate $\kl(Q|R_\lambda)$. 
We cannot compute the summation as written in Eq.~\eqref{eqn:klpr}, as we do not know the ground-truth probabilities $P(\cdot)$
and the support of a typical model distribution is prohibitively large, since it is the space of all sequences of tokens.
As a result of these two issues, \name cannot be tractably computed in closed form. 

We employ a Monte Carlo estimator using samples $\xv_i\sim P$ and $\xv_i'\sim Q$ to overcome the fact that ground-truth probabilities $P(\cdot)$ are unknown.
We circumvent the intractable support size by computing \name in a quantized embedding space that is sensitive to important features of text.

The overall estimation procedure is as follows.
First, we sample human text $\xv_i \sim P$ and machine text $\xv_i'\sim Q$. We then embed each text sequence using an external language model $M$ (e.g., GPT-2~\citep{radford2019language}) to obtain embeddings $\{M(\xv_i)\}_{i=1}^N$  and $\{M(\xv_i')\}_{i=1}^{N'}$. 
Each embedding is now a vector $M(\xv)\in \reals^d$.
Next, we jointly quantize the embedded samples (e.g. with $k$-means~\cite{manning2001foundations}), and count the cluster assignments to form histograms, giving low-dimensional discrete distributions that approximate each high-dimensional text distribution.
In particular, the distribution $P$ of human text is approximated by the discrete distribution $\tilde P$ of support size $k$, which is defined as,
\begin{align}
\label{eqn:approx}
    \tilde{P}(j)=\frac{1}{N} \sum_{i=1}^N \mathbb{I}\big(\phi(\xv_i) = j \big) \,,
\end{align}
where $\phi(\xv)\in\{1,\cdots,k\}$ returns the cluster id of $\xv$.
The model distribution $Q$ is approximated as $\tilde Q$ similarly.
Here, $\tilde P$ and $\tilde Q$ can be interpreted as piecewise constant approximations of $P$ and $Q$, similar to a histogram; see Figure~\ref{fig:mauve:quant}
for an illustration. 
Computing the divergence curve is now tractable, as each coordinate is a KL divergence between the $k$-element discrete distributions. 

To recap, our proposed measure $\textbf{\name}{(P,Q)}$ is the area under this divergence curve, providing a summary of all Type I and Type II errors through an efficient approximation designed for text generation.
Next, we discuss how \name compares to prior comparison measures for text (\S\ref{sec:background}), then present empirical results with \name (\S\ref{sec:experiments}).

\section{Related Work}
\label{sec:background}

\myparagraph{Divergence Measures for Text}
Prior measures of similarity/divergence between 
machine text and human text come in three broad categories: (a) reference-based, (b) statistics-based, and (c) language modeling.
Table~\ref{tab:nlp:auto-metrics-survey} summarizes the latter two categories, and contrasts them with \name. 

\textit{Reference-based measures} evaluate generated text with respect to a (small set of) reference text sample(s), rather than comparing full sequence distributions.
These include classical metrics for $n$-gram matching~\cite{papineni2002bleu,lin2004rouge,banerjee2005meteor},
which are designed to capture similarities in the surface form of the generated text and the human references, making them fundamentally ill-suited for open-ended generation. 
Moreover, it has been recently shown in \cite{novikova2017need} show that these classical metrics only weakly agree with human judgments.

More recent reference-based metrics are capable of comparisons in a high dimensional space~\cite{shimanaka2018ruse,zhang2020bertscore,sellam2020bleurt,clark2019sentence}, thereby capturing distributional semantics beyond superficial $n$-gram statistics.
For instance, Moverscore \cite{zhao2019moverscore} relies on the Word Mover's distance \cite{kusner2015word}, 
and is an instance of an optimal transportation distance~\cite{villani2021topics}.
It computes the minimum cost of transforming the generated text to the reference text, taking into account Euclidean distance between vector representations of $n$-grams, as well as their document frequencies.
The paradigm of reference-based measures is useful for targeted generation tasks such as translation and summarization where matching a set of references is paramount. It is, however, unsuitable for open-ended generation where there typically are several plausible continuations for each context and creative generations are desirable.

\begin{table*}[t]
\centering
\footnotesize
\begin{tabular}{p{1.2cm}p{2.8cm}p{4.2cm}p{4.1cm}}
\toprule
\bf Type & \bf Metric &   \bf Measures  &  \textbf{Approximates} \\
\midrule
\multirow{3}{1cm}{Statistics}
&  Zipf Coefficient \citep{holtzman2019curious} & Unigram rank-frequency statistics & --
\\
& Self-BLEU  \citep{zhu2018texygen} & N-gram diversity  & --
\\
& Generation Perplexity \citep{fan2018heirarchical} & Generation quality via external model $R$& $|\mathbb{E}_Q [\log R(\xv)]-\mathbb{E}_P [\log R(\xv)]|$ (a single point inside $\Ccal(P, Q)$) \\
\midrule[0.03em]
\multirow{4}{1cm}{Language Modeling} &
Perplexity & Test-set perplexity & $\mathbb{E}_P [\log Q(\xv)]$\\
& $\epsilon$-perplexity \citep{martins2020sparse} & Perplexity w/ Laplace smoothing & $\mathbb{E}_P [\tilde{Q}(\xv)]$\\
& Sparsemax Score \citep{martins2020sparse} & LM quality (sparsemax loss~\cite{martins2016softmax}) & $\mathbb{E}_P [\tilde{Q}(\xv)]$ \\
& Token JS-Div. \citep{martins2020sparse} & LM quality (JS divergence)
& $\mathbb{E}_P [\tilde{Q}(\xv)]$\\
\midrule[0.03em]
Divergence Curve & {\name (this work)} &\begin{tabular}{l}Quality \& diversity via \\ the divergence curve \end{tabular} & $\Ccal(P,Q)$ at all $\lambda$\\
\bottomrule
\end{tabular}
\caption{\small{Summary of automatic distributional metrics for evaluating open-ended text generation. 
\name provides a summary of all points along the divergence curve, rather than a single point.
The summary is based on comparisons in a joint embedding space, rather than a statistic computed independently on each distribution.
$\tilde{Q}$ informally refers to a quantity related to $Q$. 
}}
\label{tab:nlp:auto-metrics-survey}
\end{table*}

\textit{Statistics-based measures} compare
the model distribution $Q$ with respect to the human distribution $P$ on the basis of some statistic 
$T(P)$ and $T(Q)$. 
Property-specific statistics such as the amount of repetition \citep{holtzman2019curious,welleck2020neural}, verifiability \citep{massarelli2019decoding}, or termination \citep{welleck2020consistency} 
are orthogonal to \name, which provides a summary of the overall gap between $P$ and $Q$ rather than focusing on an individual property.
Another statistic is the generation perplexity \citep{fan2018heirarchical,holtzman2019curious},
which compares the perplexity of the model text $\xv \sim Q$ with that of human text $\xv' \sim P$ under an external model $R$.
By virtue of $T(\cdot)$ being a scalar, generation perplexity cannot trade-off the Type I and Type II errors
like \name. %
In fact, we show in Appendix~\ref{supp:pr} that the generation perplexity can be derived from {\em a single point} enclosed between the divergence curve and the axes.

\textit{Language modeling metrics} calculate how (un)likely
human text $\xv\sim P$ is under the model distribution $Q$, for instance, using the probability $Q(\xv)$.  
These metrics are related to a single point on the divergence curve, rather than a full summary. 
Examples include the perplexity of the test set (which is a sample from $P$) under the model $Q$ and its generalizations to handle sparse distributions~\cite{martins2020sparse}. 
Unlike \name, these metrics never see model text samples $\xv' \sim Q$, so they cannot account for how likely the model text is under the human distribution $P$. 
Moreover, they cannot be used for decoding algorithms such as beam search which do not define a token-level distribution.

Automatic metrics have been proposed for specific domains such as generation of dialogues~\cite{tao2018ruber}, stories~\cite{guan2020union}, and others~\cite{optiz2021towards}. 
They capture task-specific properties; see the surveys~\cite{celikyilmaz2020evaluation,sai2020survey}.
In contrast, \name compares machine and human text in a domain-agnostic manner.
Other related work has proposed metrics that rely on multiple samples for quality-diversity evaluation  \cite{Caccia2020Language}, and Bayesian approaches to compare the distribution of statistics in machine translation \cite{eikema2020map}.

\myparagraph{Non-automatic Metrics} 
HUSE~\citep{hashimoto2019unifying} aims to combine human judgements of Type I errors with Type II errors measured using perplexity under $Q$. 
Due to the costs of human evaluation, we consider HUSE, as well other metrics requiring human evaluation, such as single-pair evaluation, as complementary to \name, which is an automatic comparison measure.
As a separate technical caveat, it is unclear how to use HUSE for sparse $Q$ that assigns zero probability to a subset of text, which is the case with state-of-the-art decoding algorithms~\cite{holtzman2019curious,martins2020sparse}.

\myparagraph{Evaluation of Generative Models}
Evaluation of generative models is an active area of research in computer vision, where generative adversarial networks \cite{goodfellow2014generative} are commonly used.
However, metrics such as Inception Score \cite{salimans2016improved} are based on large-scale supervised classification tasks, and thus inappropriate for text generation.
The Fr\'{e}chet Distance~\cite{heusel2017gans,semeniuta2018accurate} and its unbiased counterpart, the Kernel Inception Distance \cite{bikowski2018demystifying} are both used for evaluating generative models, but unlike \name, do not take into account a trade-off between different kinds of errors between the learned and a reference distribution.
\citet{sajjadi2018assessing} and \citet{kynknniemi2019improved} both proposed metrics based on precision-recall curves.
\citet{djolonga2020precision} proposed information divergence frontiers as a unified framework emcompassing both these works as special cases. 
\name extends the above line of work, and is operationalized for open-ended text generation, applicable for data generated by large-scale neural language models. 
Complementary to this work, \citet{liu2021divergence} study the theory of information divergence frontiers, proving non-asymptotic bounds on the estimation and quantization error.

\section{Experiments}
\label{sec:experiments}

We perform three sets of experiments to validate \name.
Our first set of experiments (\S\ref{ssec:expr-identify}) examine how known properties of generated text with respect to generation length, decoding algorithm,  and model size can be identified and quantified by \name.
Next, in \S\ref{ssec:expr-approx} we demonstrate that \name is robust to various embedding strategies, quantization algorithms, and hyperparameter settings.
Finally, in \S\ref{subsec:human-eval} we find that \name correlates with human judgments.
The code as well as the scripts to reproduce the experiments are available online.\footnote{
        \url{https://github.com/krishnap25/mauve-experiments}.
}

\myparagraph{Tasks}
We consider open-ended text generation using a text completion task~\citep{holtzman2019curious,welleck2020neural} in three domains: web text, news and stories.
Each domain consists of a sequence dataset split into (context, continuation) pairs.
Given a context $\xv_{1:k}$, the task is to generate a continuation $\hat{\xv}_{k+1:T}\sim Q(\cdot\mid\xv_{1:k})$, forming a completion.
Each ground-truth completion $\xv_{1:T}$ is considered a sample from the true distribution $P$, while the completion $(\xv_{1:k},\hat{\xv}_{k+1:T})$ is considered a sample from $Q$.
The datasets, context and completion lengths, and number of completions used for each domain are shown in Table~\ref{table:expt:details}.

\begin{table}[t!]
\centering
\begin{adjustbox}{max width=0.99\textwidth}
\begin{tabular}{llclrrr}
\toprule
\bf Task Domain & \bf Model & \bf Finetuning & \bf Dataset & 
\bf \begin{tabular}{r} Prompt \\ Length \end{tabular} & 
\bf \begin{tabular}{r} Max. Generation \\ Length \end{tabular} & 
\bf \begin{tabular}{r} Number of \\ Generations \end{tabular} \\

\midrule

Web text & 
GPT-2 (all sizes) & 
 Pretrained  &
Webtext & $35$ tokens  & $1024$ tokens & $5000$
\\
News & 
Grover (all sizes)& 
Pretrained  & 
RealNews & varying & $1024$ tokens & $5000$
\\
Stories & GPT-2 medium & Finetuned &
WritingPrompts & $50$ tokens & $512$ tokens & $5000$ 
\\
\bottomrule
\end{tabular}
\end{adjustbox} \vspace*{0.5mm}
\caption{
\small Dataset and task summary. Note that $1024$ tokens
correspond to $\sim750$ words on average.} 
\label{table:expt:details}
\end{table}

\myparagraph{Models}
As the language model $\hat P(\cdot)$, we use GPT-2, a large-scale transformer~\cite{vaswani2017attention} pretrained on the web text dataset (see \cite{radford2019language}), that is representative of state-of-the-art autoregressive language models. 
As the embedding model $M(\cdot)$ we use GPT-2 Large, and compare others in \S\ref{ssec:expr-approx}.

\myparagraph{Decoding Algorithms}
We consider three common decoding algorithms: \textit{ancestral sampling} which samples directly from the language model's per-step distributions, $x_t\sim \hat P(x_t\mid \xv_{1:t})$, 
\textit{greedy decoding} which selects the most likely next token, $x_t=\arg\max_{x\in \mathcal{V}} \hat P(x\mid\xv_{1:t})$, as well as
\textit{nucleus sampling}~\citep{holtzman2019curious} which samples from top-$p$ truncated per-step distributions, $x_t\sim \hat P_{\text{nuc},p}(x_t\mid \xv_{1:t})$, which is defined as 
\[
    \hat P_{\text{nuc},p}(x_t \mid \xv_{1:t}) 
    \propto
    \begin{cases}
    \hat P_{\text{nuc},p}(x_t \mid \xv_{1:t}), &\text{if } x_t \in V_p , \\
    0, & \text{else}.
    \end{cases}
\]
Here, the top-$p$ vocabulary $V_p$ is the smallest set $V$ such that 
$\sum_{x \in V} \hat P(x \mid \xv_{1:t}) \ge p$.

We also consider an adversarial sampling procedure, designed to generate low-quality text that nevertheless matches the perplexity of human text. 
Adversarial perplexity sampling proceeds in two phases: (1) we generate the first 15\% of tokens in a sequence uniformly at random from the vocabulary, and (2) we generate the remaining tokens greedily to make the running perplexity of the generated sequence as close as possible to the perplexity of human text.

\subsection{Quantifying Properties of Generated Text}
\label{ssec:expr-identify}

To study \name's effectiveness as a measure for comparing text distributions, we first examine how \name quantifies known properties of generated text: a good measure should meet expected behavior that is known from existing research on each property.
Specifically, we investigate how \name behaves under changes in generation length, decoding algorithm, and model size.

\begin{figure*}[t]
\includegraphics[width=\linewidth]{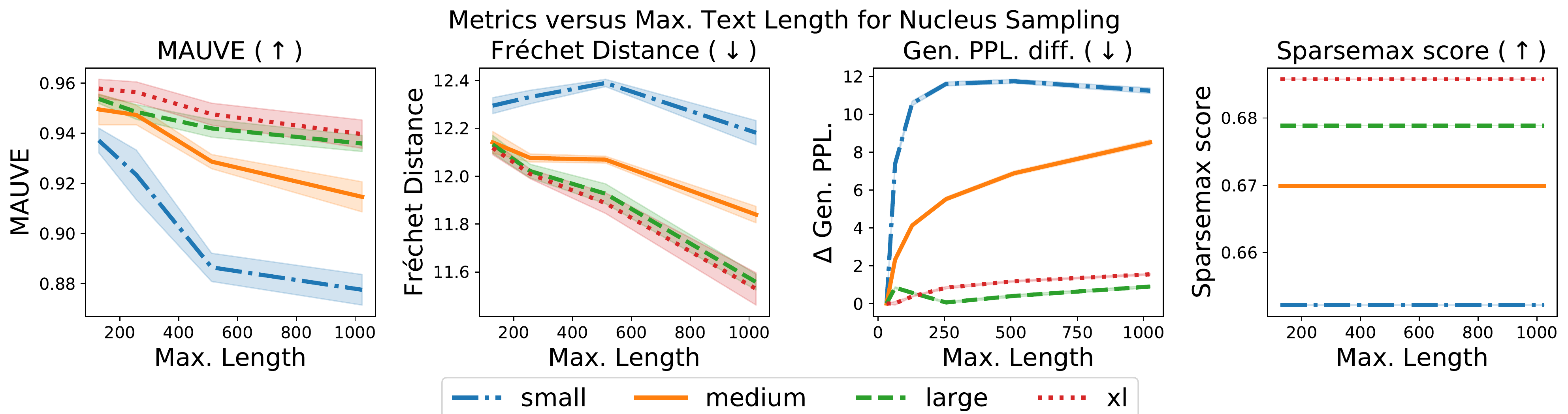}
\caption{
\small
Generation quality versus maximum generation length according to \name and three alternative measures (web text, GPT-2).
\name is the only comparison measure which identifies that generation quality decreases monotonically with increasing text length. 
The shaded area shows one standard deviation over generations from 5 random seeds.
}
\label{fig:expt:length:main}
\end{figure*}

\myparagraph{\name quantifies quality differences due to generation length}
Although large transformer-based models can generate remarkably fluent text, it has been observed that the quality of generation deteriorates with text length: 
as the generation gets longer, the model starts to wander, switching to unrelated topics and becoming incoherent~\cite{rashkin2020plotmachines}.
As a result, an effective measure should indicate lower quality (e.g. lower \name) as generation length increases. 

Figure~\ref{fig:expt:length:main} shows \name as the generation length increases, along with three alternative metrics: generation perplexity, sparsemax score, and Fr\'echet distance~\cite{heusel2017gans,semeniuta2018accurate}. 
\name reflects the desired behavior, showing a decrease in quality (lower \name) as generation length grows, with the trend consistent across model sizes. 
The other three metrics, however, show less favorable trends.
Fr\'echet distance indicates \textit{improving} quality as the length increases, while generation perplexity shows non-monotonic quality trends for the small and large models.
Finally, language modeling metrics such as the sparsemax score~\cite{martins2020sparse} remain constant, since they do not depend on the samples generated.

\myparagraph{\name identifies quality differences between decoding algorithms}
Recent work has identified two clear trends in open-ended text generation with standard autoregressive models: (1) using greedy decoding 
results in repetitive, degenerate text~\cite{holtzman2019curious,welleck2020neural,welleck2020consistency}; (2) nucleus sampling 
(and related truncated sampling methods) 
yields higher quality text than ancestral sampling~\citep{fan2018heirarchical,holtzman2019curious}.\footnote{In general this relationship depends on the nucleus hyperparameter $p$ and task. 
Here, we follow the same settings as \citet{holtzman2019curious}, and additionally include a human-assessed measure of quality.}
An effective measure should thus indicate the quality relationship $\text{greedy}\prec\text{ancestral}\prec{\text{nucleus}}$.

Table~\ref{tbl:decoding} summarizes \name's quality measures of greedy decoding, ancestral sampling, and nucleus sampling, along with alternative automated metrics and a human quality score.
\name correctly identifies the expected quality relationship, assigning the lowest quality to greedy decoding ($.016$) followed by ancestral sampling ($.882$), and the highest quality to nucleus sampling ($.940$).
Other commonly-used metrics fail to identify this relationship: generation perplexity rates the highly degenerate greedy-decoded text as better than ancestral sampling ($11.324$ vs. $19.284$), while the language-modeling metrics (SP, JS, $\epsilon$-PPL) rate nucleus-decoded text as equal to or worse than greedy decoding or ancestral sampling.
Further, as we show in Appendix~\ref{supp:expt-results}, \name rightly identifies degeneracy of beam search, thus quantifying the qualitative observations of \citet{holtzman2019curious}.
Finally, generation perplexity falls victim to the adversarial decoder (Adv.), unlike \name.\footnote{The results are consistent across model sizes and random seeds (see Appendix~\ref{supp:expt-results}).}

\myparagraph{\name quantifies quality differences due to model size}
Scaling the model size has been a key driver of recent advances in NLP, with larger models leading to better language modeling and higher quality generations in open-ended settings \citep{radford2019language,brown2020language}.
An effective metric should capture the relationship between model size and generation quality, which we verify with human quality scores.

Table~\ref{tbl:modelsize} shows \name's quality measures as the model size increases, along with alternatives and human quality scores.
\name increases as model size increases, agreeing with the human quality measure and the expectation that larger models should have higher quality generations.
The widely-used generation perplexity, however, incorrectly rates the large model's text as the best.
Although the language modeling metrics (SP, JS, and $\epsilon$-PPL) capture the size-quality relationship, they are constant with respect to length (Figure~\ref{fig:expt:length:main}), and did not correctly quantify decoding algorithm quality (Table~\ref{tbl:decoding}). 

Table~\ref{tab:mauve:expt:webtext-appendix} in Appendix~\ref{supp:expt-results} shows additional results with ancestral sampling.
In this case, human evaluators rated generations from the small model as better than those from the medium model. 
Interestingly, \name also identified this relationship, agreeing with the human ratings, in contrast to the other automatic metrics we surveyed.

\begin{table}[t]
\RawFloats
\footnotesize
\begin{minipage}[t]{0.47\linewidth}
\setlength{\tabcolsep}{2pt}
\begin{center}
{\renewcommand{\arraystretch}{1.2}%
\begin{tabular}{rrrrr}
\toprule
& \textbf{Adv.}&\textbf{Greedy} & \textbf{Sampling} & \textbf{Nucleus} \\\midrule
\textbf{Gen. PPL}($\downarrow$) & \tabemph{$\mathbf{0.05}$} & $11.3$ & $19.3$ & $1.54$\\
\textbf{Zipf}($\downarrow$)& $0.03$ & $0.02$ & $0.02$ & \tabemph{$\mathbf{0.01}$}\\
\textbf{Self-BLEU}($\downarrow$)& $0.07$ & $0.03$ & \tabemph{$\mathbf{0.02}$} & $0.03$\\
 \bf   SP($\uparrow$) & -- & $0.50$ &\tabemph{$\mathbf{0.69}$} & $0.69$\\
 \bf JS($\downarrow$) &-- & \tabemph{$\mathbf{0.35}$} & $0.37$ & $0.36$\\
 \bf $\epsilon$-PPL($\downarrow$) & --& $497$ & \tabemph{$\mathbf{11.4}$}  & $13.7$\\
\bf \name($\uparrow$) & $0.06$ & $0.02$ & $0.88$ & \tabemph{$\mathbf{0.94}$}\\\hline
 \bf Human($\uparrow$) & --& -- & $9.0$ & \tabemph{$\mathbf{15.7}$}\\
 \bottomrule
\addlinespace[0.4em] 
\end{tabular}
}
\caption{ \small 
Generation quality w.r.t\ different \textbf{decoding algorithms} (web text, GPT-2 xl) under various metrics, and humans. 
\name correctly captures the relationship $\text{greedy}\prec\text{ancestral}\prec{\text{nucleus}}$, and rates the adversarial decoder's text as low quality.
Results are consistent across model sizes and random seeds. %
Boldfaced/highlighted entries denote the best decoding algorithm under each metric.
}
\label{tbl:decoding}
\end{center}
\end{minipage}
\hfill
\begin{minipage}[t]{.47\linewidth}
\setlength{\tabcolsep}{3pt}
\begin{center}
{\renewcommand{\arraystretch}{1.2}%
\begin{tabular}{rrrrr}
\toprule
& \textbf{Small}&\textbf{Medium} & \textbf{Large} & \textbf{XL} \\\midrule
\textbf{Gen. PPL}($\downarrow$) & $11.2$ & $8.5$ & \tabemph{$\mathbf{0.9}$} & $1.5$\\
\textbf{Zipf}($\downarrow$) & $0.06$ & \tabemph{$\mathbf{0.00}$} & $0.02$ & $0.01$\\
\textbf{Self-BLEU}($\downarrow$) & $0.05$ & \tabemph{$\mathbf{0.02}$} & $0.03$ & $0.03$\\
 \bf   SP($\uparrow$) &  $0.65$ & $0.67$ & $0.68$ & \tabemph{$\mathbf{0.69}$}\\
 \bf JS($\downarrow$) & $0.41$ & $0.39$ & $0.37$ & \tabemph{$\mathbf{0.36}$}\\
 \bf $\epsilon$-PPL($\downarrow$) & $25.9$ & $18.8$ & $14.9$ & \tabemph{$\mathbf{13.7}$}\\
 \bf \name($\uparrow$) & $0.878$ & $0.915$ & $0.936$ & \tabemph{$\mathbf{0.940}$}\\\hline
 \bf Human($\uparrow$) & $-15.9$ & $-3.4$ & $12.6$ & \tabemph{$\mathbf{15.7}$}\\
 \bottomrule
\addlinespace[0.4em] 
\end{tabular}
}
\caption{\small
Generation quality w.r.t\ different \textbf{model sizes} (web text, nucleus sampling) under various metrics, as well as human evaluators.
\name captures the relationship between model size and generation quality, agreeing with human-evaluated quality.
Results are consistent across random seeds and decoding algorithms. %
Boldfaced/highlighted entries denote the best model size under each metric.
}
\label{tbl:modelsize}
\end{center}
\end{minipage}
\end{table}

\myparagraph{Summary}
\name identifies properties of generated text that a good measure should capture, related to length, decoding algorithm, and model size.
In contrast, commonly used language modeling and statistical measures did not capture all of these properties.
Unlike these alternatives, which capture a single statistic or relate to a single point on the divergence curve, \name's summary measure incorporates type I errors that quantify the degenerate text produced by greedy decoding (recall Figure~\ref{fig:main:illustration}), while capturing distribution-level information that describes quality changes from generation length, model size, and the nuanced distinction between ancestral and nucleus sampling.

\subsection{Approximations in \name}
\label{ssec:expr-approx}

\name summarizes the divergence between two text distributions with an approximation that relies on two components: an embedding model $M(\xv)$ and a quantization algorithm $\mathcal{A}$ (\S\ref{sec:mauve}, Eq.~\eqref{eqn:approx}). 
We study the effects of these two components.

\myparagraph{\name works with alternative embedding models}
Figure~\ref{fig:mauve:feature_type} (left) shows that \name with features from RoBERTa- large~\cite{liu2019roberta} gives qualitatively similar trends across model size and decoding as \name with features from GPT-2 large.
Quantitatively, the Spearman rank correlation between them across all model and decoders is $0.993$.
We observe that RoBERTa penalizes smaller models more than GPT-2 but rates greedy decoding higher. We leave further study of inductive biases in the different embedding models to future work. %

\myparagraph{\name is robust to quantization}
We compare different three different quantization algorithms:%
\begin{enumerate}[itemsep=0cm,leftmargin=0.5cm,topsep=0cm,label=(\alph*)]
 \item $k$-Means: We cluster the hidden representations using $k$-means, and represent
    them by their cluster membership
    to get a discrete distribution with
    size equal to the number of clusters.
\item  {Deep Residual Mixture Models (DRMM)}: 
    As a generalization of $k$-means, we train a deep generative model known as DRMM~\cite{hamalainen2020deep}. 
    We convert the soft clustering returned by DRMM into a hard clustering by assigning each point to its most likely cluster,
    and quantize the data using the cluster membership.
    We use DRMM with $3$ layers and $10$ components per layer for a total of $10^3$ clusters, and train it for $20$ epochs.

\item {Lattice Quantization}: We learn a $4$-dimensional feature representation of the vectors $M(\xv)$ using a deep network which maintains the neighborhood structure of the data while encouraging the features to be uniformly distributed on the unit sphere~\cite{sablayrolles2018spreading}. 
We quantize the data on a uniform lattice into $744$ bins. 
\end{enumerate}

We compare different choices of the quantization to $k$-means with $k=500$, which is our default. 
The Spearman rank correlation between \name computed with $k$-means for $k$ ranging from $100$ to $5000$ correlates nearly perfectly with that of $k=500$. In particular, the Spearman correlation is exactly $0.99$ or $1.00$. Likewise, \name computed with DRMM or lattice quantization has a near-perfect Spearman correlation of at least $0.99$ with $k$-means. 
While the actual numerical value of \name could vary with the quantization algorithm, these results show that the {\em rankings induced by various variants of \name are nearly identical}.

\begin{figure*}
    \centering
    \includegraphics[width=0.48\linewidth]{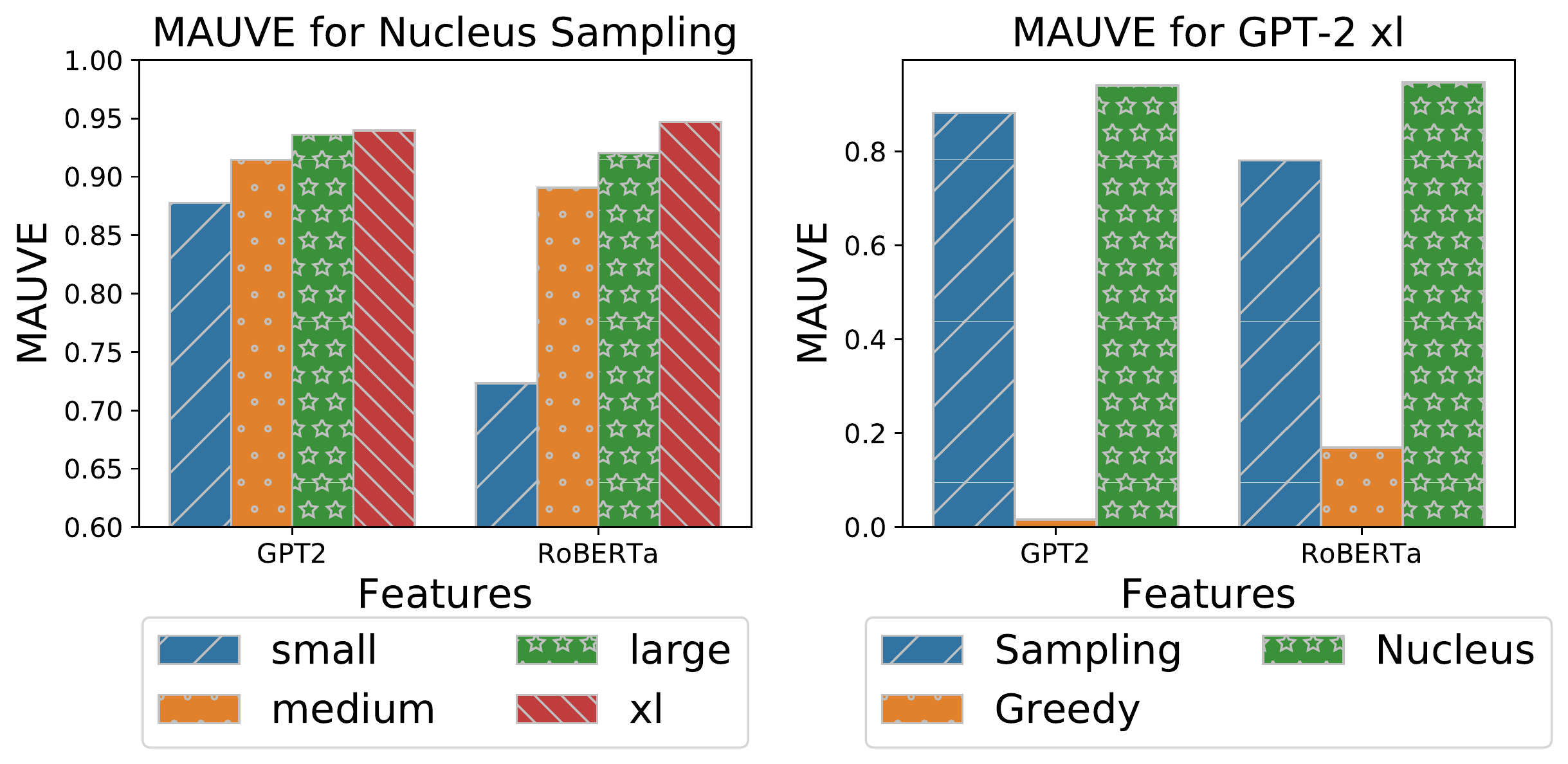}
    \includegraphics[width=0.475\linewidth]{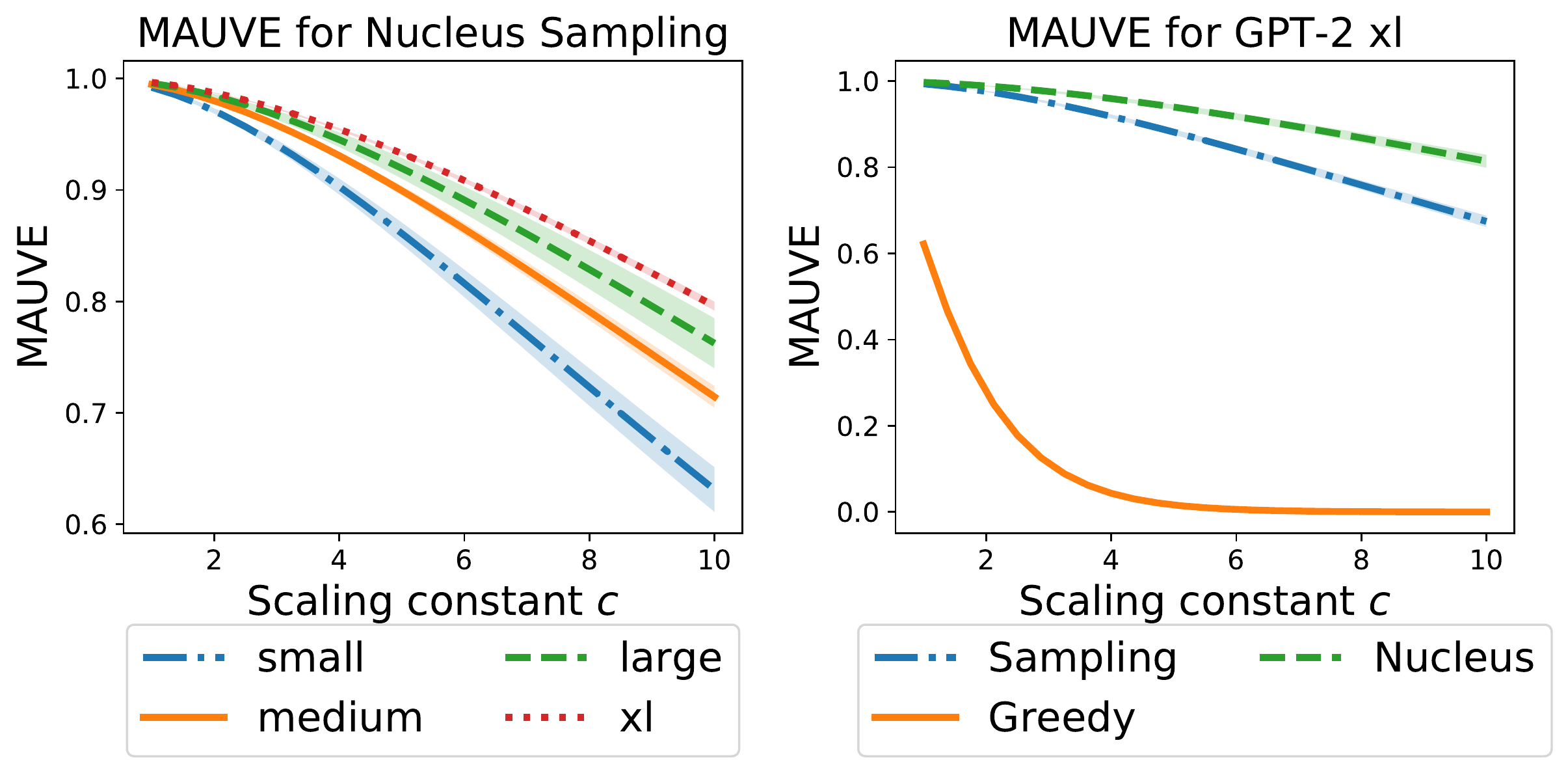}
    \caption{
    \small
    \textbf{Left}: \name computed using GPT-2 (default) and RoBERTa \cite{liu2019roberta} embeddings, across model sizes and decoding algorithms; see Table~\ref{tab:mauve:expt:bert-features-appendix} in the Appendix for further results.
    The Spearman rank correlation between the two is \textbf{0.993} across model sizes and decoding algorithms.
    \textbf{Right}: Effect of the scaling constant $c$ on \name. 
    Choice of $c$ does not affect the relative order of the curves but only the numerical value. 
    We use $c=5$ to get interpretable values with both nucleus and greedy decoding.
    }
    \label{fig:mauve:feature_type}
\end{figure*}

\myparagraph{Practical recommendation for scaling parameter}
Figure~\ref{fig:mauve:feature_type} (right) shows the effects of adjusting the scaling parameter $c$, which does not affect the relative order of the divergence curve, but adjusts the numerical value returned by \name.
As a practical recommendation, we found $c=5$ to yield interpretable values.

\subsection{Correlation with Human Judgments}
\label{subsec:human-eval}

An effective metric should yield judgments that correlate highly with human judgments, assuming that human evaluators represent a gold-standard.\footnote{Concurrent work has shown that human evaluation might not always be consistent \cite{clark2021all,karpinska2021perils}; however human judgments continue to be the gold standard for evaluating open-ended text generation.}
We evaluate how \name's quality judgments correlate with human quality judgments.
In our study, a quality judgment means choosing a particular (model, decoder) setting based on the resultant generations.

\myparagraph{Evaluation Protocol}
To obtain human judgments, we employ a pairwise setup: at each round, an annotator receives a context and continuations from two different (model, decoder) settings, and selects the continuation they found more natural using a 5-point Likert scale.
Our interface for collecting annotations is shown in Figure~\ref{fig:mauve:human-eval-interface} of Appendix~\ref{sec:a:human-eval}, which also includes further details and additional results.

We collect these annotations for web text generation with 8 different (model, decoder) settings plus a ninth setting for human-written continuations. Each setting is a GPT-2 model size paired with either ancestral or nucleus sampling. This gives us a total of $36$ pairs of settings.
Given the known difficulties with human evaluation of longer texts \cite{ippolito2020automatic}, we use a maximum completion length of $256$ tokens.
We obtain $90$ preference ratings for each pair of settings, coming from a total of $214$ crowd-workers from the Amazon Mechanical Turk platform.
The evaluators were paid USD $0.40$ per evaluation based on an estimated wage of USD $16$ per hour.

We convert these pairwise preferences to a ranking by fitting a Bradley-Terry model~\cite{bt:book:1995}, a parametric model used to predict the outcome of a head-to-head comparison. In particular, we obtain a score $w_i$ for each setting $i$ so that the log odds of humans preferring setting $i$ to setting $j$ in a head-to-head comparison is given by the difference $w_i - w_j$.
For a given comparison measure, 
we compute the Spearman rank correlation between the comparison measure and the fitted Bradley-Terry coefficients $w_i$ for each of the (model, decoder) settings.
The end result is a correlation score in $[-1,1]$, with higher values meaning that quality judgments using the comparison measure correlate with quality judgments made by human evaluators.

\myparagraph{\name correlates with human judgments}
Table~\ref{tab:mauve:expt:gold-correlation-main} shows the correlation between human judgments and five automatic evaluation metrics obtained using our evaluation protocol on the web text domain.
\name correlates highly with human judgments of how human-like ($0.952$), interesting $(0.810)$, and sensible $(0.857)$ the machine text is.
\name's correlations with human judgments are substantially higher than those for the other automated measures; for instance, the commonly-used generation perplexity has correlations that are $0.12$ to $0.17$ lower than \name's.
The results suggest that \name may act as an effective, automatic surrogate for costly human judgments.

\begin{table*}[t!]
\centering
\begin{adjustbox}{max width=0.9\textwidth}
\begin{tabular}{llrrrrrr}
\toprule
\textbf{Metric} & \textbf{Task} & \textbf{Gen. PPL} & \textbf{Zipf Coef.} & \textbf{REP} & \textbf{Distinct-4} & \textbf{Self-BLEU} & \textbf{\name} \\
\midrule
       Human-like/BT &      Web text &           $0.810$ &             $0.833$ &     $-0.167$ &             $0.738$ &            $0.595$ &        \tabemph{$\mathbf{0.952}$} \\
      Interesting/BT &      Web text &           $0.643$ &             $0.524$ &     $-0.143$ &             $0.524$ &            $0.405$ &        \tabemph{$\mathbf{0.810}$} \\
         Sensible/BT &      Web text &           $0.738$ &             $0.690$ &     $-0.071$ &             $0.595$ &            $0.524$ &        \tabemph{$\mathbf{0.857}$} \\
       \% Disc. Acc. &          News &           $0.468$ &             $0.595$ &      $0.792$ &             $0.653$ &            $0.516$ &        \tabemph{$\mathbf{0.956}$} \\
        \% Disc. Acc. &       Stories &           $0.643$ &             $0.643$ &      $0.250$ &             $0.750$ &            $0.857$ &        \tabemph{$\mathbf{0.893}$} \\
\bottomrule
\end{tabular}
 \end{adjustbox}
\caption{
Correlation of various similarity measures with human judgments when available, and the accuracy of a trained discriminator otherwise. 
``BT'' denotes the Bradley-Terry score for a pairwise human evaluation (\S~\ref{subsec:human-eval}). Boldfaced/highlighted numbers indicate highest correlation in each row. 
We observe that \name has the highest correlation with human evaluation and discriminator accuracy.
}
\label{tab:mauve:expt:gold-correlation-main}
\end{table*}

\myparagraph{\name correlates with learned discriminators}
We also measure the quality of generations by how well a trained model (a discriminator) can distinguish between real and generated text~\cite{lopezpaz2017revisiting}. 
We report the test accuracy of a binary classifier trained to discriminate between machine and human text; a lower discrimination accuracy implies that the generation is harder to distinguish from human text.
We report the accuracy of 
Grover mega as the discriminator for the news generations
as it produced the highest discrimination accuracy~\cite{zellers2019grover}
while we use GPT-2 large for the story domain.
As seen in Table~\ref{tab:mauve:expt:gold-correlation-main}, \name correlates the highest with the discrimination accuracy ($0.96$ for news and $0.89$ for stories) among all comparison measures. 
Computing the discrimination accuracy for each (model, decoder) pair requires fine-tuning a separate model, which is particularly expensive for large models such as Grover-mega. 
\name, on the other hand, does not require any training.

\section{Conclusion}
\label{sec:discussion}

We presented \name, an automatic measure of the gap between neural text and human text for open-ended text generation.
\name measures the area under a divergence curve, formalizing and summarizing a spectrum of errors that capture phenomena present in machine and human-generated text.
\name correlated with human judgments and identified quality differences due to generation length, decoding algorithm, and model size, which prior metrics struggle to capture.
Automated metrics have driven advances in computer vision and many other machine learning domains. 
\name's principled foundation and strong empirical performance offers a similar path forward for open-ended text generation systems.
Extensions of \name to closed-ended tasks, such as summarization and translation, where generations must be compared to a fixed set of gold-standard references, are promising directions for future work.  %

\paragraph{Broader Impacts Statement}

\name rewards model text which resembles human-authored text. 
However, we acknowledge the risks of rewarding systems that try to mimic humans \cite{bender2021parrots}, which is the ultimate goal of open-ended text generation.
While our research is important for developing better language generators, we also encourage the community to pay attention to the development of technology that can reliably distinguish between human and machine text.
We leave the extension of our method towards building such systems to future work.

\paragraph{Acknowledgments}
Part of this work was done while Zaid Harchaoui was visiting the Simons Institute for the Theory of Computing, and while John Thickstun was at the University of Washington. This work was supported by NSF DMS-2134012, NSF CCF-2019844, NSF DMS-2023166, the DARPA MCS program through NIWC Pacific (N66001-19-2-4031), the CIFAR ``Learning in Machines \& Brains'' program, a Qualcomm Innovation Fellowship, and faculty research awards. 

\bibliographystyle{abbrvnat}
\bibliography{anthology,acl2021,custom}

\clearpage
\appendix

\addcontentsline{toc}{section}{Appendix} %
\part{Appendix} %
\parttoc %
\clearpage

\section{Divergence Curves and Mauve: Additional Details}
\label{supp:pr}

We discuss some aspects of the divergence curves alluded to in \S\ref{sec:mauve} and \S\ref{sec:background}. 
In particular, we discuss the following.
\begin{itemize}
    \item Appendix~\ref{supp:pr:pareto}: the Pareto optimality of the divergence curves, mentioned in a footnote in \S\ref{sec:mauve}.
    \item Appendix~\ref{supp:pr:genppl}: the connection between generation perplexity and the divergence curves as mentioned in \S\ref{sec:background}.
    \item Appendix~\ref{supp:pr:quantization}: a formal definition of the quantization which is first introduced in \S\ref{sec:mauve}, as well as an illustration.
    \item Appendix~\ref{supp:pseudocode}: the pesudocode for \name.
\end{itemize}

\subsection{Pareto Optimality of Divergence Curves} \label{supp:pr:pareto}
Here, we show the property of Pareto optimality of $\Ccal(P, Q)$.
We refer to the textbook~\cite{kobayashi2007mathematics} for more background on information theory and KL divergence.
The main property we will show in this section is the following. 

\begin{proposition} \label{prop:div-pareto-opt}
    Consider two distributions $P, Q$
    with finite support and a scaling
    constant $c > 0$. 
    Let $R_\lambda$ be such that 
    $\big(\e^{-c\, \kl(Q|R_\lambda)}, 
    \e^{-c\, \kl(P|R_\lambda)}\big) \in \Ccal(P, Q)$.
    Then, $R_\lambda$ is Pareto-optimal for 
    the pair of objectives
    $\big(\kl(Q|\cdot), \kl(P | \cdot)\big)$.
    In other words, there does {\em not} exist any distribution $R$ such that 
    $\kl(Q|R) < \kl(Q|R_\lambda)$
    and 
    $\kl(P|R) < \kl(P|R_\lambda)$ simultaneously.
\end{proposition}

\begin{proof}
Let $\Fcal(P, Q)$ be the Pareto frontier 
of $\big(\kl(Q|\cdot), \kl(P | \cdot)\big)$.
The convexity of $\kl(Q|\cdot), \kl(P | \cdot)$
allows us to compute the Pareto frontier $\Fcal(P, Q)$ 
exactly by minimizing linear combinations of the objectives.
Concretely, we have from~\cite[Thm. 3.4.5, 3.5.4]{miettinen2012nonlinear} that
\begin{align*}
    \Fcal(P, Q) = 
    \Big\{
        \big(\kl(P|R_\lambda^\star),
        \kl(P|R_\lambda^\star)\big) \,:\,
        \lambda \in [0, 1]
    \Big\}
\end{align*}
where
\[
    R_\lambda^\star \in 
    \argmin_R\{ \lambda\, \kl(Q|R) + 
    (1-\lambda) \kl(P|R) \} \,.
\]
We invoke the next lemma to 
show that $R_\lambda^\star = \lambda P + (1-\lambda) Q$ to complete the proof.
\end{proof}

\begin{lemma} \label{lem:tech:KL:3pt_weighted}
	Let $P, Q, S$ be discrete distributions with finite support.
	For any $\lambda \in [0, 1]$ and $\bar \lambda = 1- \lambda$, 
	letting $R_\lambda = \lambda P + \bar \lambda Q$, 
	we have the identity
		\begin{align*}
		\lambda \, \kl(P|S) + \bar \lambda \, \kl(Q|S)
		= 
		\lambda \, \kl(P| R_\lambda) + \bar\lambda\,  \kl(Q|R_\lambda)
			+ \kl(R_\lambda|S) \,.
	\end{align*}
	Consequently, we have that 
	\[
	R_\lambda \in 
	\argmin_S \left\{\lambda \, \kl(P|S) + \bar \lambda\, \kl(Q|S) \right\} \,.
	\]
\end{lemma}
\begin{proof}
	By adding and subtracting 
	$\sum_{i} R_{\lambda, i} \log(R_{\lambda, i})$, we get,
		\begin{align*}
		\lambda \, \kl(P|S) + \bar \lambda \, \kl(Q|S) 
		&= \sum_{i}  \lambda P_i \log P_i + \bar\lambda  Q_i \log Q_i -
				R_{\lambda, i} \log S_i  \\
	    &= \sum_i  \lambda P_i \log \frac{P_i}{R_{\lambda, i}} 
	        + \bar \lambda Q_i \log\frac{Q_i}{R_{\lambda, i}} + R_{\lambda, i} \log\frac{R_{\lambda, i}}{S_i} \\
	    &= \lambda \, \kl(P| R_\lambda) + \bar\lambda \, \kl(Q|R_\lambda)
			+ \kl(R_\lambda|S) \,.
	\end{align*}
	The first two terms are independent 
	of $S$ and the last term is minimized 
	at $S = R_\lambda$.
\end{proof}

\myparagraph{Connection to Divergence Frontiers~\citep{djolonga2020precision}}
The Pareto frontier $\Fcal(P, Q)$ 
of $\big(\kl(Q|\cdot), \kl(P | \cdot)\big)$ 
(defined in the proof of Proposition~\ref{prop:div-pareto-opt})
coincides exactly with the notion 
of the {\em inclusive divergence frontier}, 
as defined by \citet{djolonga2020precision}.
It follows that the 
inclusive KL divergence frontier 
is related to the divergence curve we have defined as, 
\begin{align*}
    \Fcal(P, Q) 
    = \Big\{  &
        \left(
        c^{-1} \log t_1^{-1},
        c^{-1} \log t_2^{-1}
        \right) 
        \, : \, 
          (t_1, t_2) \in \Ccal(P, Q)
    \Big\} \,.
\end{align*}

\subsection{Generation Perplexity and Divergence Curves}
\label{supp:pr:genppl}
Recall that the generation perplexity of a text distribution $P$ is the perplexity of this distribution under an external language model $R$. That is, 
\[
    T_{\text{ppl}}(P) = \exp\left( - \expect_P[\log R(\xv)] \right) \, .
\]
For simplicity, we write the perplexity using base $\e$ rather than base $2$. 
Then, the difference in generation perplexity between $P$ and $Q$ is given by 
\begin{align*}
    \big|T_{\text{ppl}}(P) - T_{\text{ppl}}(Q)\big| 
    &= \Big|\exp\big(- \expect_P[\log R(\xv)]\big) - \exp\big(- \expect_Q[\log R(\xv)] \big) \Big| \\
    &= \Big|\exp\big(H(P) + \kl(P | R)\big) - \exp\big(H(Q) + \kl(Q | R)\big) \Big| \,,
\end{align*}
where $H(P) = -\expect_P[\log P(\xv)]$ is the Shannon entropy of $P$. 
When $H(P) = H(Q) = \log C$, i.e., both $P$ and $Q$ are equally diverse, then 
\[
     \big|T_{\text{ppl}}(P) - T_{\text{ppl}}(Q)\big|
     = C\,  \Big|\exp\big(\kl(P | R)\big) - \exp\big(\kl(Q | R)\big) \Big| \,.
\]
When $R = \lambda P + (1-\lambda)Q$, this is proportional to the difference between the reciprocal of two coordinates of {\em one} point on the divergence curve.
When $R$ is some other model, then $\big(\exp(-\kl(Q| R)), \exp(-\kl(P|R))\big)$ corresponds to the coordinates of a point enclosed within the divergence curve and the coordinate axes. Indeed, this is because the divergence curve encodes the Pareto frontier of $(\kl(P | \cdot), \kl(Q| \cdot))$.

When $H(P) \neq H(Q)$, the difference in the generation perplexity can be written as a function of 
some point $\big(\exp(-\kl(Q| R)), \exp(-\kl(P|R))\big)$ that is enclosed within the divergence curve and the axes:
\[
     \big|T_{\text{ppl}}(P) - T_{\text{ppl}}(Q)\big|
     =  \Big| C_1 \, \exp\big(\kl(P | R)\big) - C_2 \, \exp\big(\kl(Q | R)\big) \Big| \,,
\]
where $C_1 = \exp(H(P))$ and $C_2 = \exp(H(Q))$.

\subsection{Quantization: Definition} \label{supp:pr:quantization}

We formally define the quantization of a distribution. 

Consider a distribution $P$ over some space $\Xcal$. Consider a partition $S = (S_1, \cdots, S_k)$ of $\Xcal$, i.e., $\cup_{j=1}^k S_j = \Xcal$ and $S_i \cap S_j = \varnothing$ if $i\neq j$.
Quantizing the distribution $P$ over partitions $S$ gives us a multinomial distribution $\tilde P_S$ over $k$ elements. Concretely, we have, 
\[
    \tilde P_S(j) = P(S_j) \, .
\]
This histogram is a classical example of a quantizer. 
While the quantized distribution $\tilde P_S$ is a discrete multinomial distribution, it can be viewed as a piecewise constant approximation to $P$, similar to the histogram. This is visualized in Figure~\ref{fig:mauve:quant} for a two-dimensional example.
In our setting, $\Xcal$ is the space of  encoded representation of text, i.e., a Euclidean space $\reals^d$. We use data-dependent quantization schemes such as $k$-means and lattice quantization of a learned feature representation.
In one-dimension, quantization is equivalent to computing a histogram. Hence, we casually use the term ``bin'' to refer to a partition. 

\subsection{Pseudocode for \name} 
\label{supp:pseudocode}

Algorithm~\ref{alg:mauve} shows the pseudocode for computing \name. It consists of the following steps:
\begin{itemize}
    \item The first step is to embed the sampled text using an external language model $M$. In our experiments, we use GPT-2 large~\cite{radford2019language}.
    \item The second step is to quantize the embeddings. We primarily use $k$-means, which returns the cluster memberships $C_P$ and $C_Q$. 
    \item The third step is to form the quantized distributions from the cluster memberships from \eqref{eqn:approx}. This amounts to counting the number of points in each cluster contributed by $P$ and $Q$. 
    \item The next step is to build the divergence curve. The full divergence curve  \eqref{eq:div-curve} is a continuously parameterized curve for $\lambda \in (0, 1)$. For the sake of computation, we take a discretization $\Lambda$ of $[0, 1]$:
    \ifnum \value{arxiv}>0 
    {
    \begin{align} \label{eq:div-curve:appendix}
    \hat\Ccal(P, Q) =
    \Big\{ \big(\exp(-c\, \kl(Q|R_\lambda)), \exp(-c\, \kl(P|R_\lambda)) \big)
    \, :\, 
    R_\lambda = \lambda P + (1-\lambda)Q,\,
    \lambda \in \Lambda
    \Big\} \,.
    \end{align}
    }
    \else
    {    
    \begin{align} \label{eq:div-curve:appendix}
    \hat\Ccal(P, Q) =
    \Big\{ \big(\exp(-c\, \kl(Q|R_\lambda)), \exp(-c\, \kl(P|R_\lambda)) \big)
    \, :\, 
    \begin{matrix}
    R_\lambda = \lambda P + (1-\lambda)Q, \\
    \lambda \in \Lambda
    \end{matrix}
    \Big\} \,.
    \end{align}
    }
    \fi
    We take a uniform grid $\Lambda = \{1 / n, 2/n, \cdots, (n-1)/n\}$ with $n$ points.
    \item The last step is to estimate the area under $\hat \Ccal(\tilde P, \tilde Q)$ using numerical quadrature. 
\end{itemize}

\begin{algorithm}[t]
\DontPrintSemicolon
\KwInput{
Human text $\{\xv_i^P\}_{i=1}^N$, model text $\{\xv_i^Q\}_{i=1}^{N'}$, 
number of clusters $k$, embedding model $M$, discretization $\Lambda$ of $[0, 1]$.}
\KwOut{$\name(P, Q)$.}
\tcp*[l]{Embed the samples}
$\{M(\xv_i^P)\}_{i=1}^N$, $\{M(\xv_i^Q)\}_{i=1}^{N'}\gets \texttt{embed}\left(M, \{\xv_i^P\}_{i=1}^N, \{\xv_i^Q\}_{i=1}^{N'} \right)$ \vspace{0.2em}\;
\tcp*[l]{Cluster embeddings jointly}
$C_P, C_Q = \texttt{quantize}\left(\{M(\xv_i^P)\}_{i=1}^N, \{M(\xv_i^Q)\}_{i=1}^{N'}\right) $\vspace{0.2em}\\
\tcp*[l]{Form quantized distributions by counting cluster assignments}
$\tilde P \leftarrow \texttt{count}(C_P)/N$,\,\, $\tilde Q \leftarrow \texttt{count}(C_Q)/N'$\vspace{0.2em}\;
\tcp*[l]{Build the divergence curve}
Compute $\hat \Ccal(\tilde P, \tilde Q)$
from \eqref{eq:div-curve:appendix}
for $\lambda \in \Lambda$
\vspace{0.2em} \;
\tcp*[l]{Compute \name using numerical quadrature}  
\Return{$\texttt{area}\left(\hat\Ccal(\tilde P, \tilde Q)\right)$}
\caption{Pseduocode to compute \name}
\label{alg:mauve}
\end{algorithm}

\section{Software Package}\label{supp:software}

We illustrate the use of the accompanying Python package, available on GitHub\footnote{
\url{https://github.com/krishnap25/mauve}}
and installable via pip\footnote{
\url{https://pypi.org/project/mauve-text/}}
as \texttt{pip install mauve-text}.

\begin{lstlisting}[caption=Compute \name from text]
from mauve import compute_mauve

p_text = ... # list of strings representing human distribution P
q_text = ... # list of strings representing model distribution Q

# Obtain feature representation, quantize it and then compute MAUVE
out = compute_mauve(p_text=p_text, q_text=q_text, 
                    device_id=0,  # use GPU 0 for featurization
                    max_text_length=256 # truncate text to 256 tokens
                    )
print('MAUVE(P, Q) =', out.mauve)

# Plot the divergence curve
import matplotlib.pyplot as plt  
plt.plot(out.divergence_curve[:, 0], out.divergence_curve[:, 1])

# Visualize quantized versions of P and Q
import numpy as np
idxs = np.argsort(out.p_hist)[::-1]
sample_p = np.random.multinomial(n=1000, pvals=out.p_hist[idxs])
sample_q = np.random.multinomial(n=1000, pvals=out.q_hist[idxs])

x = np.arange(out.p_hist.shape[0])
plt.bar(x, sample_p, color='blue', alpha=0.3, label='P')
plt.bar(x, sample_q, color='red', alpha=0.3, label='Q')
plt.legend()

\end{lstlisting} %
\section{Experiments: Setup}
\label{supp:expt_details}

Here, we provide the full details of the experiments in \S\ref{sec:experiments}. In particular, the outline of this appendix is as follows.
\begin{itemize}
    \item Appendix~\ref{sec:supp:domains}: the three task domains considered in the expeirments.
    \item Appendix~\ref{sec:supp:training_decoding}: training and decoding hyperparameters for each of these tasks.
    \item Appendix~\ref{sec:supp:mauve-hyperparams}: hyperparameters of \name.
    \item Appendix~\ref{sec:supp:automatic-details}: details of other automatic comparison measures we consider.
    \item Appendix~\ref{sec:supp:expt:misc-details}: other details (software, hardware, running time, etc.).
\end{itemize}

\subsection{Task Domains} 
\label{sec:supp:domains}

We consider an open-ended text generation task under three domains: web text, news and stories.
As summarized in Table~\ref{table:expt:details}, we follow a slightly different setting for the task in each domain:

\myparagraph{Web text Generation}
The goal of this task is to generate articles from the publicly available analogue of the Webtext dataset\footnote{
\url{https://github.com/openai/gpt-2-output-dataset}
} 
using pretrained GPT-2 models for various sizes.
At generation time, we use as prompts the first $35$ tokens of each of the $5000$ articles from the Webtext test set, keeping maximum generation length to $1024$ tokens (which corresponds, on average, to around $750$ words).
For comparison with human text, we use the corresponding human-written continuations from the test set (up to a maximum length of $1024$ tokens).

\myparagraph{News Generation}
Under this task, the goal is to generate the body of a news article, given the title and metadata (publication domain, date, author names).
We use a Transformer-based \cite{vaswani2017attention} causal language model, Grover \cite{zellers2019grover}, which is similar to GPT-2, but tailored to generating news by conditioning on the metadata of the article as well. 
Our generations rely on pretrained Grover architectures of various sizes.
The generation prompt comprises the headline and metadata of $5000$ randomly chosen articles from the April2019 set 
of the RealNews dataset~\cite{zellers2019grover}, and the maximum article length was $1024$ tokens.
We reuse the publicly available Grover generations\footnote{
available at \url{https://github.com/rowanz/grover/tree/master/generation_examples}
} for our evaluation.

\myparagraph{Story Continuation}
Given a situation and a (human-written) starting of the story as a prompt,
the goal of this task is to continue the story. Here, we use a GPT-2 medium model fine-tuned for one epoch on the WritingPrompts dataset~\cite{fan2018heirarchical}. 
We use as generation prompts the first $50$ tokens of $5000$ randomly chosen samples of the test set of WritingPrompts.
The machine generations are allowed to be up to $512$ tokens long. The 
corresponding test examples, truncated at $512$, tokens are used as human-written continuations.

\subsection{Training and Decoding Hyperparameters}
\label{sec:supp:training_decoding}

We use size-based variants of Transformer language models \cite{vaswani2017attention} for training each task (domain). 
At decoding time, we explore a text continuation setting, conditioned on a prompt containing human-written text.
All experiments were built using pretrained (and if applicable, finetuned) models implemented in the HuggingFace Transformers library \cite{wolf2020transformers}.
The tasks are summarized in Table~\ref{table:expt:details}.

\myparagraph{Story Continuation Finetuning}
We finetune GPT-2 medium 
on the training set of the WritingPrompts dataset using the cross entropy loss 
for one epoch over the training 
set with an effective batch size of $32$
and a block size of $512$. 
We use the default optimizer and learning rate schedules of 
the HuggingFace Transformers library, i.e., the Adam optimizer with a learning rate of $5\times 10^{-5}$.

\myparagraph{Decoding Hyperparameters}
We consider pure sampling (i.e., ancestral sampling from the
model distribution), greedy decoding (i.e., choosing the argmax token recursively), and nucleus sampling~\cite{holtzman2019curious}
with parameter $p \in \{0.9, 0.92, 0.95, 0.99\}$ for web text generation
and story continuation, and $p \in \{0.9, 0.92, 0.94, 0.96, 0.98\}$ for 
news generation.

\subsection{\name Hyperparameters} \label{sec:supp:mauve-hyperparams}
\name's hyperparameters are the scaling constant $c$,
the embedding model $M$,
and the quantization algorithm (including the size of the quantized distribution). 

\subsubsection{Scaling Constant}
Note that \name's dependence on $c$ is order-preserving
since the map $x \mapsto \exp(-cx)$ is strictly monotonic in $x$. 
That is, if $\name_{c_1}(P, Q_1) > \name_{c_1}(P, Q_2)$,
then it holds that $\name_{c_2}(P, Q_2) > \name_{c_2}(P, Q_2)$
for all scaling constants $c_1, c_2 > 0$. In other words, 
the choice of the scaling constant affects the numerical value of \name 
but leaves the relative ordering between different models unchanged. 
We choose $c=5$ throughout because it allows for a meaning comparison between the numerical values of \name; Appendix~\ref{supp:expt-results-approx}
gives the values of \name for various values of $c$.

\subsubsection{Embedding Model}
We compute text embeddings from 
the GPT-2 large model. 
We find in Appendix~\ref{supp:expt-results-approx} that feature representations obtained from other large transformer models such as RoBERTA~\cite{liu2019roberta}
also achieves similar results. 

\subsubsection{Quantization}
We experiment with three quantization algorithms. 

\myparagraph{\name-$k$-means}
We first run PCA on the 
data matrix obtained from concatenating 
the
hidden state representations
of the human text and model text. 
We keep $90\%$ of the explained variance 
and normalize each datapoint to have unit 
$\ell_2$ norm.
We then run $k$-means with FAISS
for a maximum of $500$ iterations 
for $5$ repetitions; 
the repetition with the best objective value is used for the quantization. We quantize the human text distribution and the model text distribution by a histogram obtained from cluster memberships.
We vary the number of clusters in 
$\{100, 250, 500, 1000\}$.
Too few clusters makes the distributions seem closer than they actually are while too many clusters leads to many empty clusters (which makes 
all distributions seem equally far away). 
Yet, we find in Appendix~\ref{supp:expt-results-approx} 
that \name with all these values of $k$ correlate strongly with each other; we use as default $k=500$ clusters as it is neither too small 
nor too large.

\myparagraph{\name-DRMM}
We use the code released by the authors of \cite{hamalainen2020deep}.%
\footnote{
\url{https://github.com/PerttuHamalainen/DRMM}
}
We take $10$ components per layer 
and $3$ layers for a total of $1000$ components. 
We train the DRMM for $20$ epochs using the hyperparameters
suggested by the authors, i.e., 
a batch size of $64$
with a learning rate
\[
\gamma_t = \gamma_0 \min\{1,(2-2t/T)^2\}\,,
\]
where $T$ is the total number of updates
and the initial learning $\gamma_0 = 0.005$.
That is, the learning rate is set to a constant for the 
first half of the updates and then 
annealed quadratically.
For more details, 
see \cite[Appendix C]{hamalainen2020deep}.

\myparagraph{\name-Lattice}
We use the code provided by 
the authors of \cite{sablayrolles2018spreading}.\footnote{
\url{https://github.com/facebookresearch/spreadingvectors}}
We train a $4$-dimensional feature 
representation of the hidden states for 
for $200$ epochs using 
the triplet loss of 
\cite{sablayrolles2018spreading},
so that the learnt feature representations
are nearly uniformly distributed. 
We use a $2$-layer multilayer
perceptron with batch normalization 
to learn a feature representation.
We train this MLP for $200$ epochs with hyperparameters
suggested by the authors, i.e., 
a batch size of $64$ and an initial
learning rate of $0.1$. 
The learning rate is cut to $0.05$ 
after half the training and $0.01$ 
after $75\%$ of the training. 

The learnt feature representations 
are then
quantized using the lattice spherical
quantizer into $744$ bins.
This work as follows:
let $S_r$ denote the integral points 
of the unit sphere of radius $r=\sqrt{50}$ in
$\reals^4$.
A hidden state vector $x$ is run through
the trained MLP $f$ to get its feature
representation $f(x)$. 
Next, $f(x)$ is quantized to
$
\argmin_{u \in S_r} \left\| f(x) - {u}/{r}\right\|_2^2 
$.

\subsection{Automatic Comparison Measures: Details and Hyperparameters} \label{sec:supp:automatic-details}
We now describe the other automatic comparison measures we compared \name to, as well as their hyperparameters.
\begin{itemize}[itemsep=0cm,leftmargin=0.5cm]
    \item \textbf{Generation Perplexity (Gen. PPL.)}: We compute the perplexity of the generated text under the GPT-2 large model. 
    \item \textbf{Zipf Coefficient}: we report the slope of the best-fit line 
        on log-log plot of a rank versus unigram frequency plot. Note that the Zipf coefficient only depends on unigram count statistics and is invariant to, for instance, permuting the generations. We use the publicly available implementation of \cite{holtzman2019curious}.\footnote{ \url{https://github.com/ari-holtzman/degen/blob/master/metrics/zipf.py}}
    \item \textbf{Repetition Frequency (Rep.)}: The fraction of generations which devolved into repetitions. Any generation which contains at least two contiguous copies of the same phrase of any length appearing at the 
    end of a phrase is considered a repetition. We consider repetitions at the token level. 
    \item \textbf{Distinct-$n$}: The fraction of distinct $n$-grams from all possible $n$-grams across all generations. We use $n = 4$.
    \item \textbf{Self-BLEU}: Self-BLEU  is  calculated  by  computing  the  BLEU  score  of  each  generations against all other generations as references. We report the Self-BLEU using $4$-grams. 
    This operation is extremely expensive, so we follow the protocol of \cite{holtzman2019curious}: sample $1000$ generations and compute the BLEU against all other $4999$ generations. A lower Self-BLEU score implies higher diversity. This operation takes around $7$ hours to compute on a single core of an Intel i9 chip (see hardware details in the next subsection). 
    \item \textbf{Discriminator Accuracy}: We train a binary classifier to classify text as human or not. A smaller discrimination accuracy means that model text is harder to distinguish from human text. A separate classifier is trained for each model and decoding algorithm pair. For the story continuation task, we 
    train a classification head on a frozen GPT-2 large model
    using the logistic loss.
    We use $25\%$ of the data as a test set and the rest for training;
    a regularization parameter is selected with $5$-fold cross validation. For the news dataset, we follow the protocol of \cite{zellers2019grover}, i.e., a Grover mega model 
    finetuned with a binary classification head. Results with other 
    discriminators are reported in Appendix~\ref{supp:expt-results}. 
\end{itemize}

\subsection{Miscellaneous Details} \label{sec:supp:expt:misc-details}

\myparagraph{Software}
We used Python 3.8, PyTorch 1.7 and HuggingFace Transformers 4.3.2.

\myparagraph{Hardware}
All the experiments requiring a GPU (finetuning, sampling generations and computing embeddings) were performed on 
a machine with $8$ Nvidia Quadro RTX GPUs ($24$G memory each)
running CUDA 10.1.
Each only used one GPU at a time.
On the other hand, non-GPU jobs such as computation of \name and 
self-BLEU were run on a workstation with 
Intel i9 processor (clock speed: $2.80\text{GHz}$)  with $32$ virtual cores and $126$G of memory.

\myparagraph{Evaluation time for \name}
Computation of \name using $k$-means
with $5000$ generations
takes $1-3$ minutes on {\em a single core} of an Intel i9 CPU (clock speed: $2.80\text{GHz}$),
using cached hidden state representations from a GPT-2 large (which are available during generation).
On the other hand, 
\name-DRMM takes $1.75$ hours on 
a single CPU core
while \name-Lattice runs in about $5$ minutes on a single TITAN Xp GPU. 
\name-$k$-means and \name-DRMM can also run much faster on multiple CPU cores and can leverge GPUs although we did not use these features.  %
\section{Experiments: Additional Results} \label{supp:expt-results}

We elaborate on the results in \S\ref{sec:experiments}, including the results for the other domains. The outline is as follows.
\begin{itemize}
    \item Appendix~\ref{sec:supp:expt-size_decoding}: full results across model size and decoding (elaborating on \S\ref{ssec:expr-identify}).
    \item Appendix~\ref{sec:supp:expt-length}: full results across text length (elaborating on \S\ref{ssec:expr-identify}).
    \item Appendix~\ref{supp:expt-results-approx}: study of approximations in \name (elaborating on \S\ref{ssec:expr-approx}).
    \item Appendix~\ref{sec:supp:expt:misc}: some miscellaneous plots such use of \name for hyperparameter tuning.
\end{itemize}
Note that \S\ref{subsec:human-eval} is elaborated on in Appendix~\ref{sec:a:human-eval}.

\begin{table*}[t!]
\centering
\begin{adjustbox}{max width=\textwidth}
\begin{tabular}{lllllllrl}
\toprule
\bf GPT-2  Size    &   \bf    Decoding             &      \bf              Gen. PPL &           \bf     Zipf Coef. &               \bf        Rep. &       \bf         Distinct-4 &           \bf      Self-BLEU & \bf Human/BT($\uparrow$) & \bf \name($\uparrow$) \\
\midrule
\multirow{4}{*}{small} & Sampling &          $101.880_{0.627}$ &           $0.926_{0.001}$ &           $0.001_{0.000}$ &           $0.941_{0.001}$ &           $0.327_{0.003}$ &                      $-27.52$ &            $0.589_{0.018}$ \\
      & Greedy &                    $1.224$ &                   $1.037$ &                   $0.942$ &                   $0.072$ &           $0.465_{0.000}$ &                 --            &                    $0.008$ \\
      & Nucleus, $0.9$ &           $23.788_{0.144}$ &           $1.012_{0.002}$ &           $0.010_{0.001}$ &           $0.859_{0.002}$ &           $0.436_{0.004}$ &              $-15.78$ &            $0.878_{0.006}$ \\
      & Adversarial &            \tabemph{$\mathbf{12.554}$} &           $1.073$ &           $0.006$ &           $0.365$ &           $0.525$ &             -- &            $0.043$ \\
\midrule
\multirow{4}{*}{medium} & Sampling &          $129.263_{0.798}$ &           $0.872_{0.001}$ &           $0.001_{0.000}$ &           $0.953_{0.001}$ &           $0.281_{0.002}$ &              $-30.77$ &            $0.373_{0.010}$ \\
      & Greedy &                    $1.241$ &                   $0.978$ &                   $0.903$ &                   $0.091$ &                   $0.415$ &                 --        &                    $0.012$ \\
      & Nucleus, $0.9$ &           $21.073_{0.134}$ &           \tabemph{$\mathbf{0.957}_{0.001}$} &           $0.005_{0.001}$ &  \tabemph{$\mathbf{0.884}_{0.001}$} &  \tabemph{$\mathbf{0.402}_{0.003}$} &                  $-3.43$ &            $0.915_{0.006}$ \\
      & Adversarial &           \tabemph{$\mathbf{12.554}$} &           $1.006$ &           $0.005$ &           $0.381$ &           $0.444$ &            -- &            $0.044$ \\
\midrule
\multirow{4}{*}{large} & Sampling &           $30.080_{0.196}$ &           $0.930_{0.002}$ &  \tabemph{$\mathbf{0.002}_{0.001}$} &           $0.916_{0.001}$ &           $0.358_{0.001}$ &                $-6.93$ &            $0.845_{0.010}$ \\
      & Greedy &                    $1.232$ &                   $0.983$ &                   $0.881$ &                   $0.100$ &                   $0.413$ &       --          &                    $0.012$ \\
      & Nucleus, $0.95$ &  $13.499_{0.058}$ &  $0.967_{0.002}$ &           $0.006_{0.001}$ &           $0.870_{0.001}$ &           $0.412_{0.002}$ &             $12.55$ &            $0.936_{0.003}$ \\
      & Adversarial &           \tabemph{$\mathbf{12.554}$} &           $0.965$ &           $0.005$ &           $0.395$ &           $0.429$ &        -- &            $0.035$ \\
\midrule[0.03em]
\multirow{4}{*}{xl} & Sampling &           $31.886_{0.447}$ &           $0.930_{0.001}$ &           $0.002_{0.001}$ &           $0.913_{0.001}$ &           $0.360_{0.003}$ &            $8.97$ &            $0.882_{0.006}$ \\
      & Greedy &                    $1.278$ &                   $0.975$ &                   $0.859$ &                   $0.115$ &                   $0.417$ &               --        &                    $0.016$ \\
      & Nucleus, $0.95$ &           $14.143_{0.043}$ &           $0.966_{0.002}$ &           $0.005_{0.000}$ &           $0.868_{0.001}$ &           $0.413_{0.002}$ &      \tabemph{$\mathbf{15.66}$} &   \tabemph{$\mathbf{0.940}_{0.006}$} \\
      & Adversarial &           \tabemph{$\mathbf{12.554}$} &           $0.986$ &           $0.005$ &           $0.397$ &           $0.448$ &      -- &            $0.057$ \\
\midrule
Human &         n/a          &                   $12.602$ &                   $0.952$ &                   $0.002$ &                   $0.878$ &                   $0.382$ &         $47.25$ &            --                \\
\bottomrule
\end{tabular}
 \end{adjustbox}
\caption{Comparison measures across different model sizes, and decoding approaches for web text generations. 
Subscripts indicate the s.d. across 5 runs for the sampling-based methods; greedy decoding, being deterministic, always returns the same value for a given model.
For nucleus sampling, we show the best hyperparameter value from $\{0.9, 0.92, 0.95, 0.99\}$ as per \name.
The column ``Human/BT'' gives the Bradley-Terry score obtained from a pairwise human evaluation (\S\ref{subsec:human-eval}).
Boldfaced numbers indicate best performance according to the measure, or closest to the human reference, when applicable.
\name shows that larger models perform better, across decoding approaches; moreover, nucleus sampling is the best decoding algorithm as per \name.
}
\label{tab:mauve:expt:webtext-appendix}
\end{table*}

\begin{table*}[t!]
\centering
\begin{adjustbox}{max width=0.95\textwidth}
\begin{tabular}{llrrrrrrr}
\toprule
\bf Grover Size      &    \bf  Decoding              &       \bf     Gen. PPL &     \bf    Zipf Coef. &     \bf           Rep. &    \bf     Distinct-4 &      \bf    Self-BLEU &  \bf \% Disc. Acc.($\downarrow$) & \bf \textsc{mauve}($\uparrow$) \\
\midrule
\multirow{3}{*}{base} & Sampling &           $37.505$ &           $0.942$ &           $0.002$ &           $0.882$ &           $0.419$ &                     $99.925$ &                    $0.700$ \\
      & Greedy &            $1.413$ &           $1.038$ &           $0.518$ &           $0.081$ &           $0.548$ &                    $100.000$ &                    $0.005$ \\
      & Nucleus, $0.96$ &           $23.064$ &           $0.974$ &           $0.006$ &  \tabemph{$\mathbf{0.847}$} &           $0.462$ &                     $99.950$ &                    $0.701$ \\
\midrule[0.03em]
\multirow{3}{*}{large} & Sampling &           $27.796$ &  ${0.946}$ &  \tabemph{$\mathbf{0.002}$} &           $0.878$ &           $0.429$ &                     $99.450$ &                    $0.794$ \\
      & Greedy &            $1.575$ &           $1.012$ &           $0.366$ &           $0.124$ &           $0.504$ &                    $100.000$ &                    $0.005$ \\
      & Nucleus, $0.98$ &           $20.792$ &           \tabemph{$\mathbf{0.962}$} &           $0.002$ &           $0.859$ &           $0.450$ &                     $98.475$ &                    $0.750$ \\
\midrule[0.03em]
\multirow{3}{*}{mega} & Sampling &           $22.656$ &           $0.950$ &           $0.001$ &           $0.879$ &           $0.427$ &                     $97.300$ &                    $0.808$ \\
      & Greedy &            $1.796$ &           $1.003$ &           $0.316$ &           $0.176$ &           $0.500$ &                    $100.000$ &                    $0.005$ \\
      & Nucleus, $0.96$ &  \tabemph{$\mathbf{14.834}$} &           $0.972$ &           $0.003$ &           $0.848$ &  \tabemph{$\mathbf{0.469}$} &            \tabemph{$\mathbf{88.675}$} &           \tabemph{$\mathbf{0.813}$} \\
\midrule
Human &      n/a             &           $15.356$ &           $0.956$ &           $0.002$ &           $0.842$ &           $0.473$ &         --                     &        --                    \\
\bottomrule
\end{tabular}
 \end{adjustbox}
\caption{News generation evaluation across different Grover model sizes, and decoding approaches.
For nucleus sampling, we show the best hyperparameter value from $\{0.9, 0.92, 0.94, 0.96, 0.98\}$
as per \name.
Disc. Acc. denotes the discrimination accuracy (\%) of a Grover mega model trained to distinguish human text
from machine text generated with the model and decoding algorithm of each row. 
Boldfaced numbers indicate performance closest to the human reference when applicable, or the best performance according to the measure.
\name favors nucleus sampling over ancestral sampling and greedy decoding.
}
\label{tab:mauve:expt:grover-appendix}
\end{table*}

\begin{table*}[t!]
\centering
\begin{adjustbox}{max width=0.95\textwidth}
\begin{tabular}{llllllll}
\toprule
\textbf{Decoding} &                   \textbf{Gen. PPL} &                \textbf{Zipf Coef.} &                       \textbf{REP} &                \textbf{Distinct-4} &                 \textbf{Self-BLEU} & \textbf{\% Disc. Acc. ($\downarrow$)} & \textbf{\textsc{mauve}($\uparrow$)} \\
\midrule
Sampling          &           $38.983_{0.143}$ &  \tabemph{$\mathbf{1.066}_{0.002}$} &  \tabemph{$\mathbf{0.001}_{0.000}$} &           $0.833_{0.001}$ &           $0.518_{0.003}$ &              $0.781_{0.004}$ &            $0.905_{0.010}$ \\
Nucleus, $0.9$  &           $15.433_{0.042}$ &           $1.201_{0.002}$ &           $0.006_{0.001}$ &           $0.719_{0.001}$ &           $0.637_{0.002}$ &              $0.752_{0.004}$ &            $0.887_{0.008}$ \\
Nucleus, $0.92$ &           $17.422_{0.060}$ &           $1.179_{0.002}$ &           $0.004_{0.001}$ &           $0.742_{0.001}$ &           $0.620_{0.003}$ &              $0.720_{0.006}$ &            $0.901_{0.005}$ \\
Nucleus, $0.95$ &  \tabemph{$\mathbf{21.599}_{0.127}$} &           $1.147_{0.002}$ &           $0.003_{0.000}$ &  \tabemph{$\mathbf{0.775}_{0.002}$} &           $0.589_{0.005}$ &     \tabemph{$\mathbf{0.686}_{0.006}$} &   \tabemph{$\mathbf{0.920}_{0.004}$} \\
Top-$100$  &           $16.527_{0.041}$ &           $1.252_{0.001}$ &           $0.002_{0.000}$ &           $0.743_{0.001}$ &           $0.631_{0.001}$ &              $0.782_{0.002}$ &            $0.884_{0.007}$ \\
Top-$500$  &           $23.833_{0.076}$ &           $1.153_{0.001}$ &           $0.001_{0.000}$ &           $0.794_{0.001}$ &  \tabemph{$\mathbf{0.576}_{0.002}$} &              $0.697_{0.005}$ &            $0.919_{0.005}$ \\
Greedy            &                    $1.739$ &                   $1.362$ &                   $0.988$ &                   $0.101$ &                   $0.742$ &                      $0.997$ &                    $0.005$ \\
\midrule
Human             &                   $19.704$ &                   $1.101$ &                   $0.001$ &                   $0.783$ &                   $0.571$ &                              &                            \\
\bottomrule
\end{tabular}
 \end{adjustbox}
\caption{Story continuation evaluation across different and decoding approaches with GPT-2 medium.
Disc. Acc. denotes the discrimination accuracy (\%) of a classifier (a frozen GPT-2 large model with classification head) trained to distinguish human text from machine text generated with the decoding algorithm of each row. 
Boldfaced numbers indicate performance closest to the human reference when applicable, or the best performance according to the measure.
\name favors nucleus and top-$K$ sampling over ancestral sampling and greedy decoding.
}
\label{tab:mauve:expt:wp-appendix}
\end{table*}

\begin{table*}[t!]
\centering
\begin{adjustbox}{max width=0.7\textwidth}
\begin{tabular}{llrrrrl}
\toprule
\bf GPT-2  Size    &   \bf    Decoding             &  \bf   SP($\uparrow$) &  \bf JS($\downarrow$) & \bf $\epsilon$-PPL($\downarrow$) & \bf Human/BT($\uparrow$) & \bf \name($\uparrow$) \\
\midrule
\multirow{4}{*}{small} & Greedy &                         $0.431$ &           $0.394$ &                   $1049.589$ &          --            &                    $0.008$ \\
& Sampling &             $0.653$ &           $0.425$ &                     $19.401$ &             $-27.52$ &            $0.589_{0.018}$ \\
      & Nucleus, $0.9$ &         $0.652$ &           $0.414$ &                     $25.938$ &             $-15.78$ &            $0.878_{0.006}$ \\
\midrule
\multirow{4}{*}{medium} & Greedy &                          $0.465$ &           $0.371$ &                    $708.057$ &              --        &                    $0.012$ \\
& Sampling &      $0.670$ &           $0.402$ &                     $14.631$ &             $-30.77$ &            $0.373_{0.010}$ \\
      & Nucleus, $0.9$ &              $0.670$ &           $0.391$ &                     $18.821$ &              $-3.43$ &            $0.915_{0.006}$ \\
\midrule
\multirow{4}{*}{large} & Greedy &                     $0.483$ &           $0.359$ &                    $580.020$ &            --          &                    $0.012$ \\
& Sampling &     $0.679$ &           $0.381$ &                     $12.658$ &              $-6.93$ &            $0.845_{0.010}$ \\
      & Nucleus, $0.95$ &    $0.679$ &           $0.374$ &                     $14.938$ &              $12.55$ &            $0.936_{0.003}$ \\
\midrule[0.03em]
\multirow{4}{*}{xl} & Greedy &                      $0.496$ &  \tabemph{$\mathbf{0.349}$} &                    $497.696$ &              --        &                    $0.016$ \\
    & Sampling &      \tabemph{$\mathbf{0.686}$} &           $0.369$ &            \tabemph{$\mathbf{11.412}$} &               $8.97$ &            $0.882_{0.006}$ \\
      & Nucleus, $0.95$ &            $0.686$ &           $0.363$ &                     $13.677$ &    \tabemph{$\mathbf{15.66}$} &   \tabemph{$\mathbf{0.940}_{0.006}$} \\
      & Adversarial &           n/a &           n/a &                     n/a &             -- &            $0.057$ \\
\bottomrule
\end{tabular}
 \end{adjustbox}
\caption{\name versus comparison measures based on language modeling (SP, JS and $\eps$-PPL) across different model sizes, and decoding approaches for web text generations. 
SP, JS and $\eps$-PPL are deterministic because they do not require generations from a decoding algorithm; moreover they cannot measure the quality of the adversarial decoding.
The column ``Human/BT'' gives the Bradley-Terry score obtained from a pairwise human evaluation (\S\ref{subsec:human-eval}).
Boldfaced numbers indicate best performance according to the measure.
}
\label{tab:mauve:expt:webtext-app-LM}
\end{table*}

\begin{table}[t!]
\centering
{\small
\begin{tabular}{lcccccccc}
\toprule
\multirow{2}{*}{\bf Discriminator}
& \multicolumn{2}{c}{BERT}
& \multicolumn{3}{c}{GPT-2}
& \multicolumn{3}{c}{Grover} \\
\cmidrule(lr){2-3}
\cmidrule(lr){4-6}
\cmidrule(lr){7-9}
& Base & Large & Small & Medium & Large & Base & Large & Mega \\
\midrule
\bf Correlation & 0.803 & 0.817 & 0.831 & 0.829 & 0.822 & 0.928 & 0.956 & 0.925 \\
\bottomrule
\end{tabular}

 }
\vspace{0.5mm}
\caption{Spearman rank correlation between the discrimination accuracy
for various discriminators and \name for news generation. 
All entries have a $p$-value of $<2\times 10^{-6}$.
}
\label{tab:mauve:expt:grover-disctype}
\end{table}

\begin{table*}[t!]
\centering
\begin{adjustbox}{max width=\textwidth}
{ \small
\begin{tabular}{lccccccc}
\toprule
{\bf Decoding}
& Greedy & Beam $b=4$ & \begin{tabular}{c} Beam $b=4$ + \\  no $4$-gram repeat \end{tabular} & Beam $b=8$ & \begin{tabular}{c} Beam $b=8$ + \\  no $4$-gram repeat \end{tabular} & Ancestral & Nucleus \\
\midrule
\bf Mauve & $0.008$ & $0.021$ & $0.026$ & $0.366$ & $0.341$ & $0.589_{0.02}$ & \tabemph{} $\mathbf{0.878_{0.007}}$  \\
\bottomrule
\end{tabular} }
\end{adjustbox}
\caption{
\name and beam search: we compare beam search with beam sizes $b=4,8$ (with and without allowing $4$-gram repetitions) with other decoding algorithms of Table~\ref{tab:mauve:expt:webtext-appendix} for web text generation with GPT-2 small. The subscript denotes the standard deviation over $5$ random seeds, and is omitted for the deterministic greedy decoding and beam search.
}
\label{tab:mauve:expt:beam-search}
\end{table*}

\begin{table*}[t!]
\centering
{ \small
\begin{tabular}{llcc}
\toprule
GPT-2 size         &     Decoding               &  
RoBERTa   &  GPT-2    \\
\midrule
small & Sampling &    0.174 &  0.589 \\
         & Greedy &   0.056 & 0.008 \\
         & Nucleus, $0.9$ &    0.723 &  0.878 \\
\midrule
medium & Sampling &    0.292 &   0.372 \\
         & Greedy &    0.114 & 0.011 \\
         & Nucleus, $0.9$ &    0.891 &  0.915 \\
\midrule
large & Sampling &    0.684 &  0.845 \\
         & Greedy &    0.125 & 0.012 \\
         & Nucleus, $0.95$ &     0.920 &  0.936 \\
\midrule
xl & Sampling &     0.780 &  0.881 \\
         & Greedy &     0.170 & 0.016 \\
         & Nucleus, $0.95$ &    \tabemph{\textbf{0.947}} &  \tabemph{\textbf{0.940}} \\
\bottomrule
\end{tabular} %
}
\caption{
Comparison of \name computed with dense embeddings from RoBERTa~\cite{liu2019roberta} large
with the default GPT-2 large.
Boldfaced numbers indicate best performance according to the measure. 
The two feature representations have a Spearman rank correlation of $0.993$.
See Figure~\ref{fig:mauve:feature_type}
for a visual representation of a subset of this table.
}
\label{tab:mauve:expt:bert-features-appendix}
\end{table*}

\begin{figure*}[t]
\includegraphics[width=0.97\linewidth]{fig/length/nucleus_main.pdf}
\includegraphics[width=0.97\linewidth]{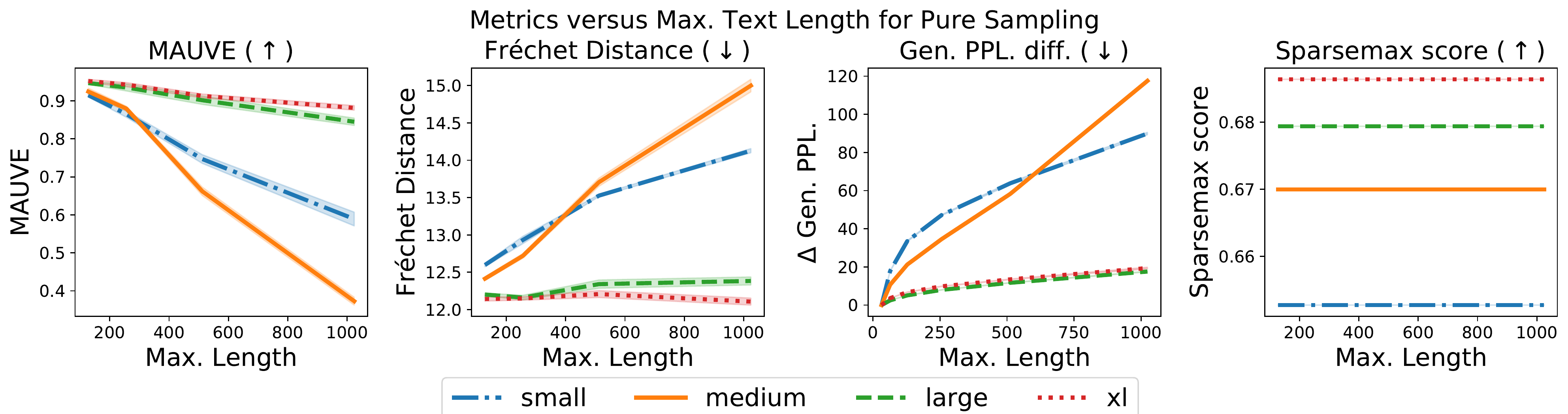}
\includegraphics[width=0.97\linewidth]{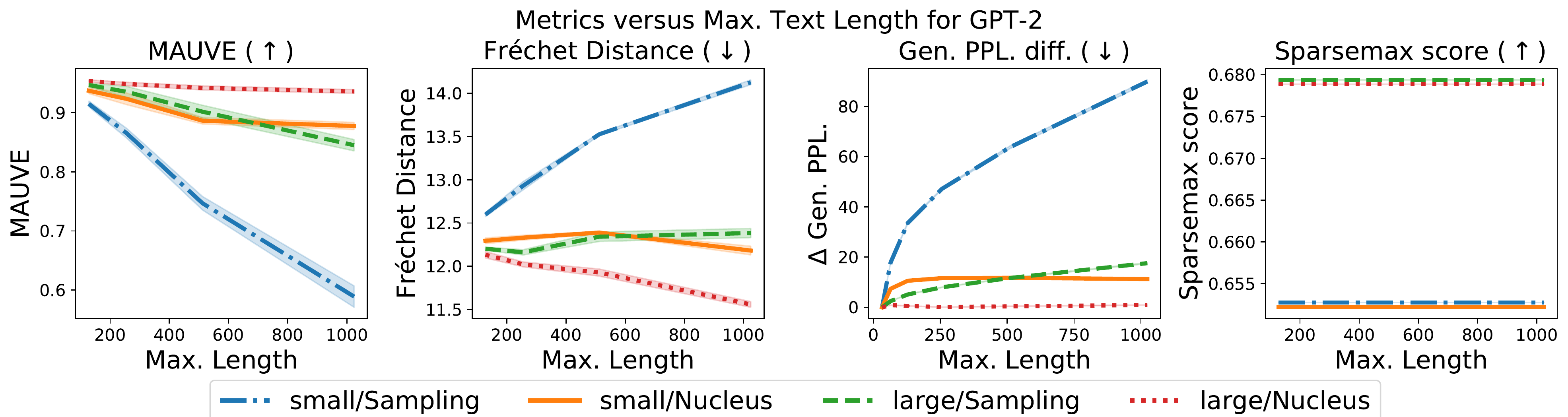}
\caption{Generation quality versus maximum generation length as per various comparison measures for 
web text generation with GPT-2. 
We expect the quality of the generation to degrade as the maximum length of the text (both machine and human-written) increases. 
\name is the only comparison measure which correctly shows this behavior across all models and decoding algorithms.
The shaded area denotes one standard deviation over generations from 5 random seeds.
}
\label{fig:a:expt:length:all}
\end{figure*}

\subsection{Comparison of Measures Across Model Size and Decoding} \label{sec:supp:expt-size_decoding}

Full versions of \autoref{tbl:decoding} and \autoref{tbl:modelsize}
can be found between \autoref{tab:mauve:expt:webtext-appendix} for statistics-based measures
and \autoref{tab:mauve:expt:webtext-app-LM} for the language modeling measures. 
The corresponding tables
for the news and story domains are
Tables~\ref{tab:mauve:expt:grover-appendix}
and~\ref{tab:mauve:expt:wp-appendix} respectively. 

\textbf{Note}: The main paper and the appendix treat the statistics-based measures differently (Gen. PPL., Zipf, Self-BLEU, etc). 
For each statistic $T$, the main paper (Tables~\ref{tbl:decoding} and~\ref{tbl:modelsize}) 
gives the difference $|T(Q)-T(P)|$ between the statistic on model text and human text, 
while in Tables~\ref{tab:mauve:expt:webtext-appendix},~\ref{tab:mauve:expt:grover-appendix},~\ref{tab:mauve:expt:wp-appendix} of the supplement, we show $T(Q)$ in the row corresponding to $Q$ and $T(P)$ in the row corresponding to human. 

\myparagraph{Results}
From Table~\ref{tab:mauve:expt:webtext-appendix},
we observe that among the decoding
approaches, nucleus sampling achieves the best
\name followed by sampling and lastly by greedy
decoding. This trend is consistent with the fraction
of distinct 4-grams. On the other hand, in comparison with the perplexity of human text, Gen. PPL
is too high for sampling and too low for greedy
decoding; it does not give us a way to directly
compare which of these two is better. \name,
however, rates greedy decoding as far worse than
ancestral sampling. This is consistent with the empirical observation that greedy decoding produces
extremely degenerate text~\cite{welleck2020neural}.
Adversarial perplexity sampling produces unintelligible text which nevertheless has perfect Gen.
PPL, thus demonstrating its unsuitability for 
as a comparison measure. 

The results in Tables~\ref{tab:mauve:expt:grover-appendix} and~\ref{tab:mauve:expt:wp-appendix} for the news and story domains are qualitatively similar to the webtext domain. 
\name, like discrimination accuracy, rates larger models as better and nucleus sampling as better than ancetral sampling and greedy decoding. An exception to this rule is Grover large, where \name thinks ancestral sampling is better than nucleus sampling. The statistics-based measures Zipf coefficient, Repetition and the fraction of distinct $4$-grams all prefer smaller Grover sizes.

Next we turn to the language modeling comparison measures
in Table~\ref{tab:mauve:expt:webtext-app-LM}. 
JS consistently favors greedy decoding,
which produces far worse text than other decoding algorithms. Likewise, $\eps$-PPL favors ancestral sampling, which also produces somewhat degenerate
text~\cite{holtzman2019curious}, while SP appears to
be unable to distinguish between ancestral sampling and nucleus sampling. This makes SP, JS and $\eps$-PPL unsuitable to compare generated text to human text.

While most measures behave nearly as expected
across model architectures (larger models produce
better generations for the same decoding algorithm), Self-BLEU prefers generations from GPT-2 medium over GPT-2 large or xl. This indicates
that while measures based on word/token statistics
are important diagnostic tools, they do not capture
the quality of generated text entirely.

\myparagraph{Discriminator Accuracy: Choice of Discriminator}
We show the Spearman rank correlation between the discriminator accuracy 
for various choices of the discriminator in  Table~\ref{tab:mauve:expt:grover-disctype}.
The results show that \name has a strong correlation with the discrimination accuracy for a variety of discriminators, including one based on a masked language model, BERT \cite{devlin2018bert}. 
This correlation is particular strong for the Grover-based discriminators. 
We note that evaluating any one model and decoding algorithm pair requires fine-tuning a separate model. This can be particularly expensive for the larger models such as Grover mega. \name, on the other hand, is inexpensive in comparison.

\myparagraph{Beam Search}
We also calculate \name for beam search in Table~\ref{tab:mauve:expt:beam-search}. \name is able to quantify the qualitative observations of \citet{holtzman2019curious}: beam search produces extremely degenerate text, but slightly better than greedy decoding. Disallowing repetition of $4$-grams substantially improves the quality of the produced text, since the most glaring flaw of beam search is that the text is highly repetitive. However, the quality of the resulting text is still far worse than produced by ancestral sampling, and hence also nucleus sampling.

\subsection{Behavior Across Text Length} \label{sec:supp:expt-length}

We now turn to the plot of comparison measures versus text length in Figure~\ref{fig:a:expt:length:all}. We expect the quality of the generation to degrade as the maximum length of the text (both machine and human-written) increases. 

\myparagraph{Comparison Measures}
Figure~\ref{fig:a:expt:length:all} plots
\name, Gen. PPL. and the Sparsemax score~\cite{martins2020sparse}. 
In addition we also plot the Fr\'echet distance, a variant of the
Fr\'echet Inception Distance (FID)~\cite{heusel2017gans}
which is the de facto standard evaluation metric for GANs in computer vision. The FID is computed as the 
Wasserstein-2 distance between Gaussians fit to the feature representation from using an Inception network; we adapt it to our setting by using embeddings from GPT-2 large instead.
For Gen. PPL., we plot the difference of Gen. PPL., i.e., 
$|T_{\text{ppl}}(Q_{\le \ell}) - T_{\text{ppl}}(P_{\le \ell})| $,
$T_{\text{ppl}}(P_{\le \ell})$ denotes the perplexity of the text $\xv\sim P$ 
truncated at a length of $\ell$. The perplexity is measured using GPT-2 large model as the external language model.

\myparagraph{Results}
\name indeed shows this expected behavior. 
However, the Fr\'echet distance~\cite{heusel2017gans} actually decreases for nucleus sampling for all GPT-2 sizes and ancestral sampling for GPT-2 xl. 
This shows that it is not suitable as an evaluation metric for text. 
 While Gen. PPL. mostly agrees with \name about quality versus text length, we observe non-monotonic behavior for nucleus sampling with GPT-2 small and large. 
Finally, sparsemax score~\cite{martins2020sparse} does not depend on the samples generated and is therefore independent of the maximum text length.

\subsection{Effect of Approximations of \name} \label{supp:expt-results-approx}
We expand upon the approximation results from the main paper in 
\S\ref{ssec:expr-approx}.

\myparagraph{Embedding Model}
Table~\ref{tab:mauve:expt:bert-features-appendix} shows \name compute with RoBERTa large in addition to the default GPT-2 large. 
We restrict the maximum text length of the RoBERTa model to $256$ BPE tokens, since RoBERTa cannot handle sequences of length $1024$ tokens.
We observe similar trends with both: larger models are rated higher and nucleus sampling is preferred over ancestral sampling while greedy decoding is rated very low. 
The Spearman rank correlation between \name computed with the two feature representations is $0.993$, indicating that \name is robust to feature representations. We observe that RoBERTa penalizes ancestral sampling more while rating greedy decoding higher across all model sizes. We leave a study of the biases induced by different feature representations to future work.

\myparagraph{Quantization Algorithm} 
We compare different choices of the quantization to $k$-means with $k=500$, which is our default. 
The Spearman rank correlation between \name computed with $k$-means for $k$ ranging from $100$ to $5000$ correlates nearly perfectly with that of $k=500$. In particular, the Spearman correlation is exactly $0.99$ or $1.00$. Likewise, \name computed with DRMM or lattice quantization has a near-perfect Spearman correlation of at least $0.99$ with $k$-means. 
While the actual numerical value of \name could vary with the quantization algorithm, these results show that the {\em rankings induced by various variants of \name are nearly identical}.

\begin{wrapfigure}{r}{0.5\textwidth} 
    \centering
    \includegraphics[scale=0.4]{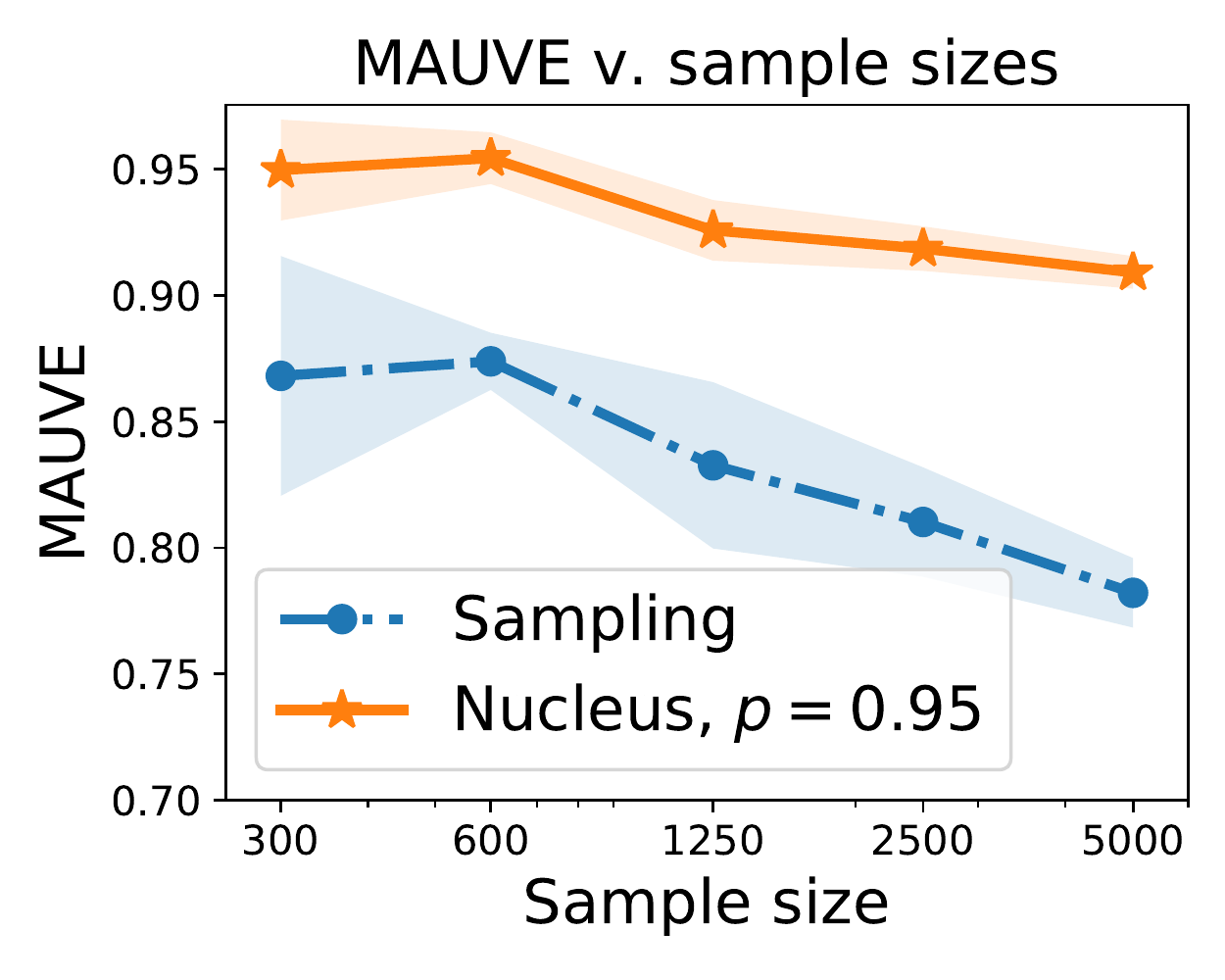}
    \caption{Effect of the sample size on \name.}
    \label{fig:mauve:num_gen_dependence}
\end{wrapfigure}

See Figure~\ref{fig:a:expt:model_selection:appendix} (Left) for how \name-$k$-means depends on the number of clusters, $k$. If $k$ is too small ($k < 100$), all methods are scored close to $1$. If $k$ is too large $k > 2000$), all methods are scored close to $0$. There is a large region between these two extremes where \name-$k$-means is effective.

\myparagraph{Effect of Number of Generations}
Figure~\ref{fig:mauve:num_gen_dependence} plots 
the value of \name versus the sample size $n$,
with the number of clusters in $k$-means
chosen as $k=n/10$. 
We observe that a smaller sample size gives an optimistic estimate of \name; this is 
consistent with \cite[Prop. 8]{djolonga2020precision}.
We also note that a smaller sample size leads to a larger variance in \name.

\subsection{Miscellaneous Plots} \label{sec:supp:expt:misc}
Figure~\ref{fig:a:expt:model_selection:appendix} plots \name for nucleus and top-$K$ sampling for various values of the hyperparameters $p$ and $K$.

\begin{figure*}[t]
\centering
\includegraphics[width=0.33\linewidth]{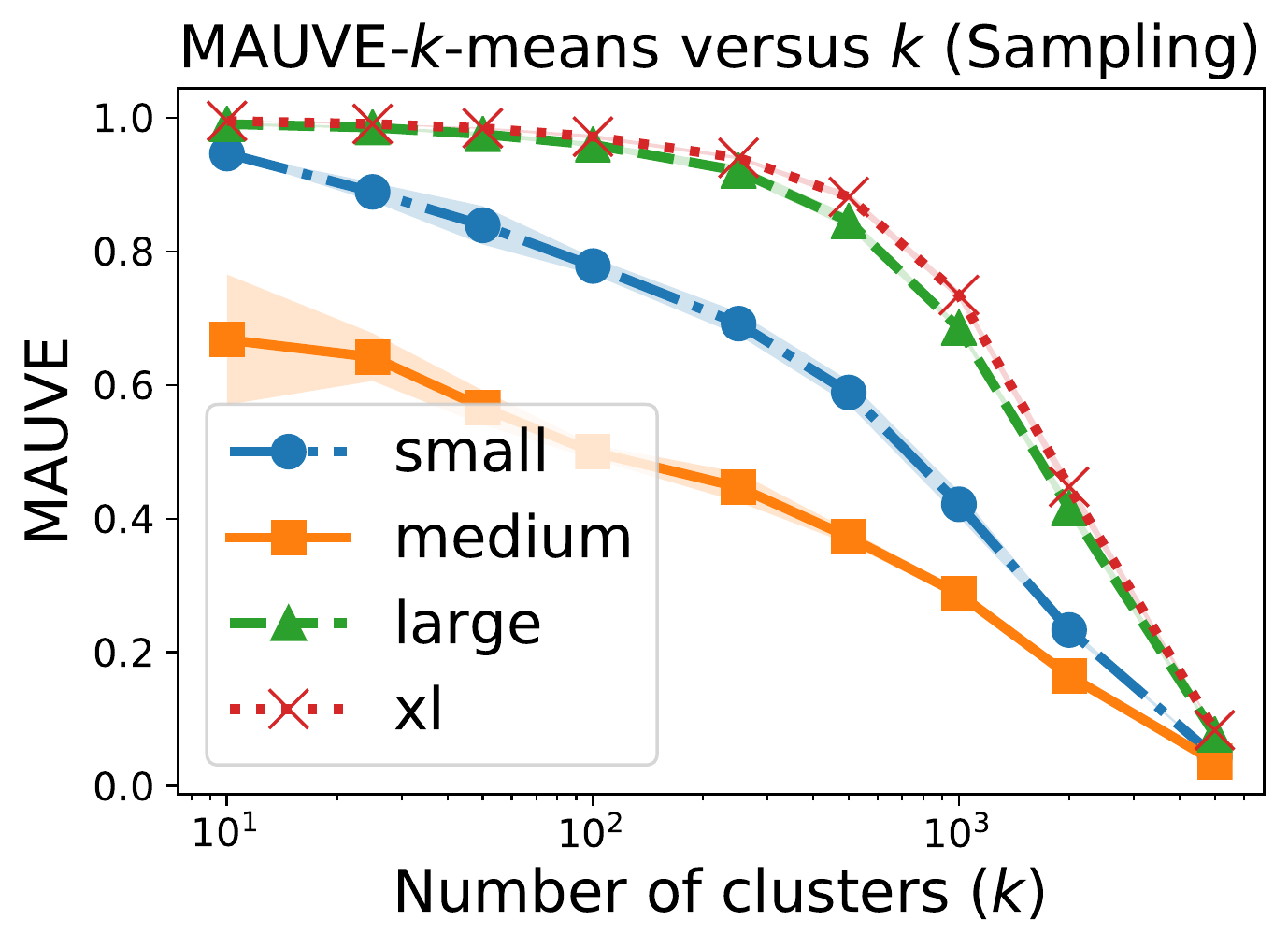}
\includegraphics[width=0.63\linewidth]{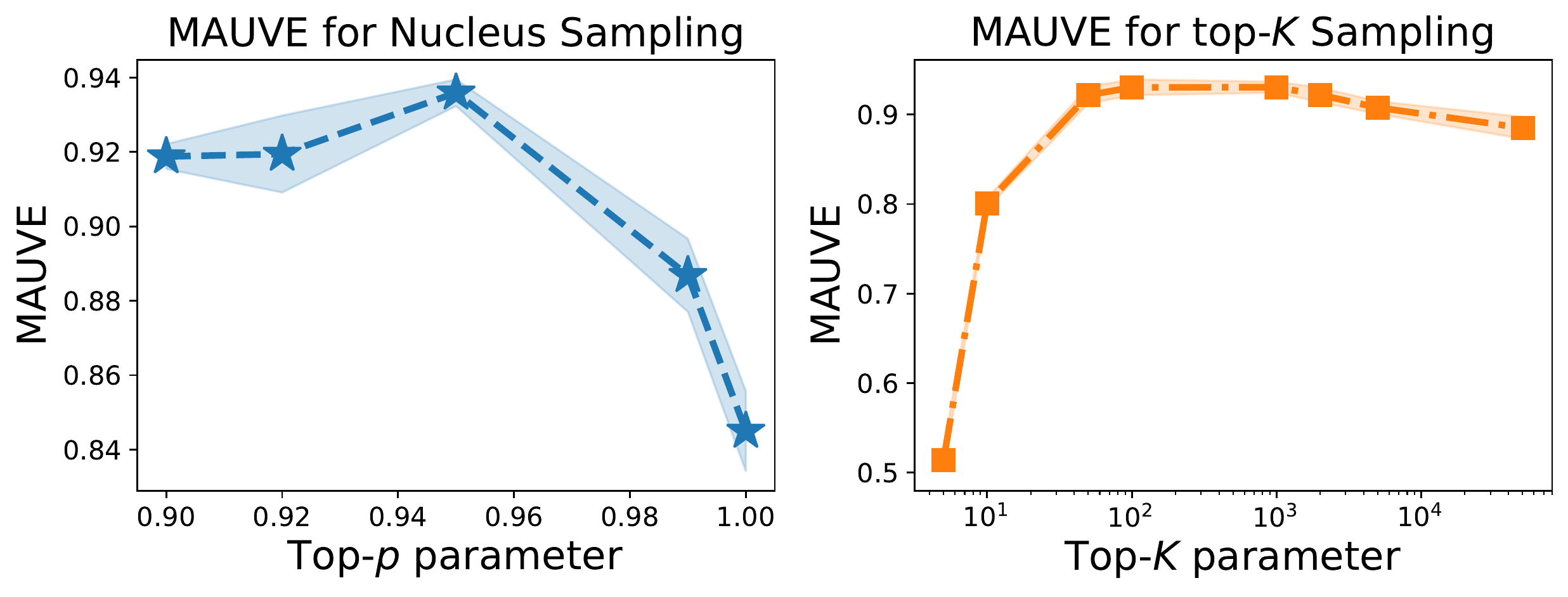}
\caption{
\textbf{Left}: \name-$k$-means for various values of the number of clusters $k$. We use $k=500$ as our default because it is neither too small (every method is scored close to $1$) nor too large (every method is scored close to $0$).
\textbf{Center \& Right}:
\name for nucleus and top-$K$ sampling 
for different values of $p$ and $K$
for GPT-2 large. 
\name rates nucleus sampling with $p=0.95$
and top-$K$ sampling with $100 \le K \le 1000$ as the best choices.
The shaded area denotes one s.d. over generations from 5 random seeds.
}
\label{fig:a:expt:model_selection:appendix}
\end{figure*} %
\section{Human Evaluation: Protocol and Full Results} \label{sec:a:human-eval}

Here, we describe the human evaluation protocol and results of \S\ref{subsec:human-eval} in detail. The outline for this section is 
\begin{itemize}
    \item Section~\ref{sec:a:human-eval:overview}: Overview of 
        the human evaluation setup.
    \item Section~\ref{sec:a:human-eval:bt}: Details of the statistical model we fit to the raw data. 
    \item Section~\ref{sec:a:human-eval:results}: Full results of the human evaluation. 
    \item Section~\ref{sec:a:human-eval:datasheet}: Additional details of the human evaluation protocol. \nocite{shimorina2021human}
\end{itemize}

\begin{figure*}[t!]
    \centering
    \fbox{\includegraphics[width=0.9\linewidth]{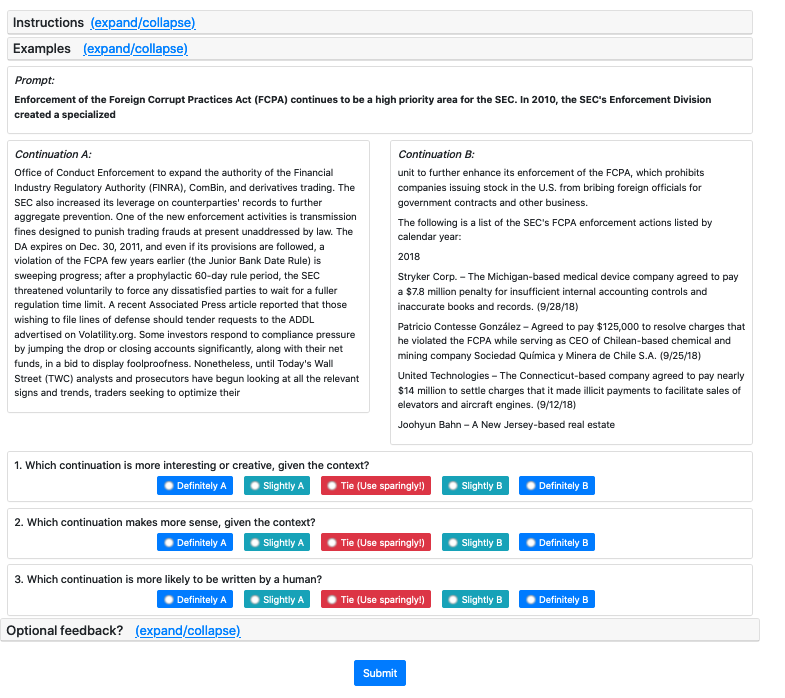}}
    \caption{Mechanical Turk interface for human evaluation.}
    \label{fig:mauve:human-eval-interface}
\end{figure*}

\subsection{Overview}
\label{sec:a:human-eval:overview}

We performed a human evaluation for web text generations 
where human annotators are instructed to select one from a pair of texts.
The pairs might come from human and machine text, or different sources of machine text; each is based on the same prompt for generation (recall that we obtained the prompt as a prefix from the human text).

The annotators were presented with a pairs of continuations of the same prompt and were instructed 
to choose which one is (a) more interesting, (b) more sensible, and, (c) more likely to be written by a human. Each question could have a different answer.

We considered all four GPT-2 model sizes with pure sampling and nucleus sampling. 
We collected $90$ annotations for each of the 
$8$ model-human pairs and ${8\choose 2}$
model-model pairs on the Amazon Mechanical Turk platform
using the interface shown in Figure~\ref{fig:mauve:human-eval-interface}.
We fit a Bradley-Terry model to obtain a ranking from 
the pairwise preferences of the crowd-workers. 
We report the correlation of \name with 
obtained Bradley-Terry scores.

\subsection{From Pairwise Preferences to Ranking: the Bradley-Terry Model}
\label{sec:a:human-eval:bt}
We compute the Bradley-Terry (BT) scores from the pairwise preferences 
obtained from the human evaluation along each of the three axes
interesting, sensible and more likely to be written by a human.

\myparagraph{Bradley-Terry Model Review}
Given $n$ players with scores $w_1, \cdots, w_n$, the 
the Bradley-Terry model~\cite{bt:book:1995} 
models the outcome of a head-to-head comparison
of any two players using a sigmoid%
\footnote{the scaling factor $100$ is arbitrary and does not change the model}
\[
    \text{Prob}(i \text{ beats } j) = \frac{1}{1 + \e^{-(w_i - w_j) / 100}} \,.
\]
The model also assumes the outcome of each head-to-head comparison of any pair of players is independent of all
other comparisons.
Note that the model is invariant to 
additive shifts of the scores, i.e., 
the model probabilities induced by scores 
$w_1 + C, \cdots, w_n + C$ is same as the that 
induced by $w_1, \cdots, w_n$ for any constant $C$.
For uniqueness, we normalize the scores so that their mean is $0$.

\myparagraph{Fitting the Model}
The Bradley-Terry model can be fit to data 
using Zermelo's algorithm~\cite{hunter2004mm}.
Suppose that we are given a dataset
of head-to-head comparisons summarized by 
numbers $N_{ij}$ denoting the number of times player $i$
has defeated player $j$. 
Then, the negative log-likelihood $\ell(w_1, \cdots w_n)$ 
of the data under the 
Bradley-Terry model can be written as
\[
    \ell(w_1, \cdots, w_n) = 
        -\sum_{i=1}^n\sum_{j=1}^n N_{ij} \log(1 + \e^{-(w_i - w_j) / 100}) \,.
\]
This is convex in the parameters $w_1, \cdots, w_n$ since 
the log-sum-exp function is convex. 
Zermelo's algorithm~\cite{hunter2004mm} can be used to compute the maximum likelihood estimate. 
Denote $\widetilde w_i = w_i / 100$. 
Starting from an initial estimate $\widetilde w_1^{(0)}, \cdots, \widetilde w_n^{(0)}$,
each iteration of Zermelo's algorithm performs the update
\[
     u_i^{(t)} = \log\left( \sum_{j\neq i} N_{ij} \right)
     - \log\left( \sum_{j\neq i} \frac{N_{ij} + N_{ji}}{
     \exp(\widetilde w_i^{(t)}) + \exp(\widetilde w_j^{(t)})} \  \right)
\]
followed by the mean normalization 
\[
    \widetilde w_i^{(t+1)} = u_i^{(t)} - \frac{1}{n} \sum_{j=1}^n u_j^{(t)} \,.
\]
\myparagraph{Processing Raw Data}
We collect the result of a head-to-head comparison using 5 options: Definitely A/B, Slightly A/B or a Tie. We combine Definitely A and Slightly A into a single category denoting that A wins, while ties were assigned to either A or B uniformly at random.

\begin{table*}[t!]
\centering
\begin{adjustbox}{max width=0.99\textwidth}
{\small
\begin{tabular}{llccc}
\toprule
      &          & BT/Human-like & BT/Interesting & BT/Sensible \\
\midrule
Human & {} &      $47.251$ &       $25.503$ &    $43.229$ \\
xl & Nucleus, $p=0.95$ &      \tabemph{$\mathbf{15.664}$} &       \tabemph{$\mathbf{23.046}$} &    \tabemph{$\mathbf{31.888}$} \\
      & Sampling &       $8.966$ &        $9.529$ &     $7.753$ \\
large & Nucleus, $p=0.95$ &      $12.553$ &        $6.785$ &     $8.781$ \\
      & Sampling &      $-6.935$ &       $-1.532$ &    $-7.106$ \\
medium & Nucleus, $p=0.9$ &      $-3.429$ &      $-12.824$ &    $-7.293$ \\
      & Sampling &     $-30.769$ &      $-34.323$ &   $-32.004$ \\
small & Nucleus, $p=0.9$ &     $-15.783$ &       $-0.697$ &    $-7.442$ \\
      & Sampling &     $-27.518$ &      $-15.487$ &   $-37.805$ \\
\bottomrule
\end{tabular}
 }
\end{adjustbox}
\caption{Fitted Bradley-Terry (BT) scores 
for each of the three axes 
rated by human annotators:
``Human-like'' denotes measures how likely 
the text is to be written by a human,
while ``Interesting'' and ``Sensible'' quantify how 
interesting or sensible the text is. 
The Spearman rank correlations between 
each of these scores are
($p$-value $\le 5\times 10^{-4}$ for each):
    Human-like and Interesting: $0.917$,
    Human-like and Sensible: $0.917$,
    Interesting and Sensible: $0.967$.
}
\label{tab:mauve:expt:webtext-human-eval-all}
\end{table*}

\begin{table*}[t!]
\centering
\begin{adjustbox}{max width=0.99\textwidth}
{\small 
\begin{tabular}{ccccccc}
\toprule
{} & Gen. PPL & Zipf Coef. &       REP & Distinct-4 & Self-BLEU & \textsc{mauve} \\
\midrule
BT/Human-like  &  $0.810$ &    $0.833$ &  $-0.167$ &    $0.738$ &   $0.595$ &        \tabemph{$\mathbf{0.952}$} \\
BT/Interesting &  $0.643$ &    $0.524$ &  $-0.143$ &    $0.524$ &   $0.405$ &        \tabemph{$\mathbf{0.810}$} \\
BT/Sensible    &  $0.738$ &    $0.690$ &  $-0.071$ &    $0.595$ &   $0.524$ &        \tabemph{$\mathbf{0.857}$} \\
\bottomrule
\end{tabular}
 }
\end{adjustbox}
\caption{Spearman rank correlation between 
the Bradley-Terry scores from the human evaluation 
and the various automatic comparison measures.
}
\label{tab:mauve:expt:webtext-human-eval-all-corr}
\end{table*}

\subsection{Full Results of the Human Evaluation}
\label{sec:a:human-eval:results}

\myparagraph{BT Model for Human Eval}
In our setting, each ``player'' is
a source of text, i.e., 
one human, plus, 
eight model and decoding algorithm pairs
(four model sizes GPT-2 small/medium/large/xl
coupled with pure sampling or nucleus sampling). 
We compute the BT score of each player 
as the maximum likelihood estimate of 
corresponding the parameters $w_1,\cdots, w_n$
based on head-to-head human evaluation data.

A higher BT score indicate a stronger preference 
from human annotators.
The BT scores are reported in Table~\ref{tab:mauve:expt:webtext-human-eval-all}.
The Spearman rank correlations between 
each of these scores are
($p$-value $\le 5\times 10^{-4}$ for each):
\begin{itemize}[itemsep=0cm,leftmargin=0.5cm]
    \item Human-like and Interesting: $0.917$,
    \item Human-like and Sensible: $0.917$,
    \item Interesting and Sensible: $0.967$.
\end{itemize}

\myparagraph{Interpreting BT scores}
The BT scores reported in Table~\ref{tab:mauve:expt:webtext-human-eval-all}
give us predictions from the sigmoid model above.
For example, consider the column ``BT/Human-like''.
The best model-generated text, 
GPT-2 xl with nucleus sampling, will 
lose to human text with probability $0.578$.
At the other end, GPT-2 small with nucleus sampling 
will lose to human text with probability $0.679$.
This shows that there is still much room for improvement in 
machine generated text. 

\myparagraph{Discussion}
In general, the BT scores from human evaluations 
and \name both indicate that 
(a) nucleus sampling is better than pure sampling for the same model size, and,
(b) larger model sizes are better for the same decoding algorithm.
There is one exceptions to this rule, as per both
the human evaluations and \name: 
GPT-2 small is better than GPT-2 medium for pure sampling.

\myparagraph{Correlation Between Comparison Measures}
We compare the Spearman rank correlation between the various automatic comparison measures and the BT scores from human evaluations in Table~\ref{tab:mauve:expt:webtext-human-eval-all-corr}.
In terms of being human-like, 
we observe that \name correlates the best ($0.95$) with human evaluations. 
While this is also the case for Zipf coefficient, we note that it is based purely on unigram statistics; it is invariant to the permutation of tokens, which makes it unsuitable to evaluate generations.

We note that \name does disagree with human evaluations on specific comparisons.
For instance, \name rates nucleus sampling with GPT-2 medium as being better than pure sampling from GPT-2 large and xl. The same is also the case with Gen. PPL.
We leave a detailed study of this phenomenon to future work.

\subsection{Additional Details}
\label{sec:a:human-eval:datasheet}

We describe more details for the human evaluation. The terminology below is taken from \cite{shimorina2021human}.

\myparagraph{Number of Outputs Evaluated}
We compare $9$ players: 
one player is ``human'', representing human-written text, 
whereas the other $8$ are text generated by the model using 
the first $35$ tokens of the corresponding human generation as a prompt. Each of the $8$ non-human players come from 
a GPT-2 model of different sizes (small, medium, large, xl)
and two decoding algorithms (pure sampling and nucleus sampling). 
We perform $90$ comparisons between each pair of players, 
so each player is evaluated $90 \times 8 = 720$ times. 

\myparagraph{Prompt Filtering}
We manually selected $1831$ out of $5000$ prompts which are well-formed English sentences from the webtext test set\footnote{
The webtext dataset is scraped from the internet 
and is {\em not} curated. It contains poor prompts
such as headers of webpages or error message, such as:
``Having trouble viewing the video? Try disabling any ad blocking extensions currently running on your browser''
or ``Front Page Torrents Favorites My Home My Galleries Toplists Bounties News Forums Wiki''.
We exclude such prompts as they are unsuitable for human evaluation.
}.
For every head-to-head comparison, we sample $90$
prompt without replacement and then sample the corresponding 
completions (for human-generated text, we use the test set of webtext). We only consider a pair of players for human evaluation if the generation from each player is 
at least $200$ BPE tokens long (and we truncate each generation
at a maximum length of $256$ BPE tokens).

\myparagraph{Number of Evaluators}
 $214$ unique evaluators 
participated in the evaluation. 
Of these, 
$11$ evaluators supplied at least $50$ annotations
$95$ evaluators supplied at least $10$
annotations.

\myparagraph{Evaluator Selection and Pay}
We conduct our human evaluation on Amazon Mechanical Turk. Since the task only requires elementary reading and understanding skills in English, we open 
the evaluations to non-experts. Each crowd-worker was paid 
$0.40$ per annotation. The pay was estimated based on a 
$\$16$/hour wage for the $85$\textsuperscript{th} percentile
of response times from a pilot study (which was approx. $98$ seconds per annotation). There evaluators are not previously known to the authors.

\myparagraph{Training and Instructions}
The evaluators were given instructions about the task 
and two detailed examples. No other training was provided due to the elementary nature of the task. The screenshots of these 
examples are given in Figure~\ref{fig:mauve:human-eval-examples} 
while the instructions read:
\begin{displayquote}
{\small

\textbf{Task Info}: We are studying how good AI models are at generating text on the internet. You are given a snippet of text from a random document on the internet, called the "prompt" or the "context", as well as and two continuations, A and B. One or both of these is written by an AI. You must choose (a) which of two continuations is more interesting, (b) which makes more sense given the prompt, and, (c) which is more likely to have been written by a human, as per your assessment.

\textbf{Guidelines}:
\begin{itemize}[itemsep=0cm,leftmargin=0.5cm]
    \item There are five choices for each question: Definitely A/B, Slightly A/B, or Tie. Please use the "Tie" option extremely sparingly! (No more than one in every ten pairs should be chosen as a tie along any of the three questions).
    \item The questions can have different answers! Some text is very creative or interesting, but it doesn't quite fit the prompt or make sense.
    \item Try to focus on quality over quantity. The text can be long but contain rambly gibberish.
    \item Don't worry if the text ends abruptly, or has other artifacts of the website downloading process (text like 'Advertisement' for instance).
    \item Please do your best, some of these are pretty challenging!
    \item Answering each question should take around 1.5 minutes on average, as per our estimation. We have calibrated the pay to be $\$16$ per hour with this speed.
\end{itemize}
}
\end{displayquote}

\myparagraph{Quality Control}
All annotations made in under $25$ seconds were excluded
for quality control
(the mean response time per annotation was $47$ seconds).

\myparagraph{Quality Criteria}
We use three quality criteria. The questions asked to the evaluators are (verbatim):
\begin{enumerate}[itemsep=0cm,leftmargin=0.5cm]
\item Interestingness: ``Which continuation is more interesting or creative, given the context?"
\item Sensible: ``Which continuation makes more sense, given the context?''
\item Human-like: ``Which continuation is more likely to be written by a human?''
\end{enumerate}
Note that we do explicitly name the criteria in the evaluation form, although those names could be inferred from the definitions. We use these names only in the paper. 

Further Details:
\begin{itemize}[itemsep=0cm,leftmargin=0.5cm]
\item Each of the criteria is a ``{Goodness}'' criteria as per the classification of \cite{belz2020disentangling}. Goodness refers to the setting where there is no single, general mechanism for deciding when outputs are maximally good, only for deciding for two outputs which is better and which is worse. E.g.\ for Fluency, even if outputs contain no disfluencies, there may be other ways in which any given output could be more fluent.
\item Each criterion assesses outputs as a whole, not just form or just content.
\item The output quality is assessed without referring to anything other than the output itself, i.e.\ no  system-internal or external frame of reference. 
\item Each criterion involves a subjective assessments of preferences by evaluators.
\item The quality of outputs is assessed \textit{without} considering their \textit{effect} on something external to the system, e.g.\ the performance of an embedding system or of a user at a task.
\item For each criteria, we provide 5 options: ``Definitely/Slightly A/B'' and ``Tie (Use sparingly!)''
\end{itemize}

\begin{figure*}[p!]
    \centering
    \includegraphics[width=0.65\textwidth]{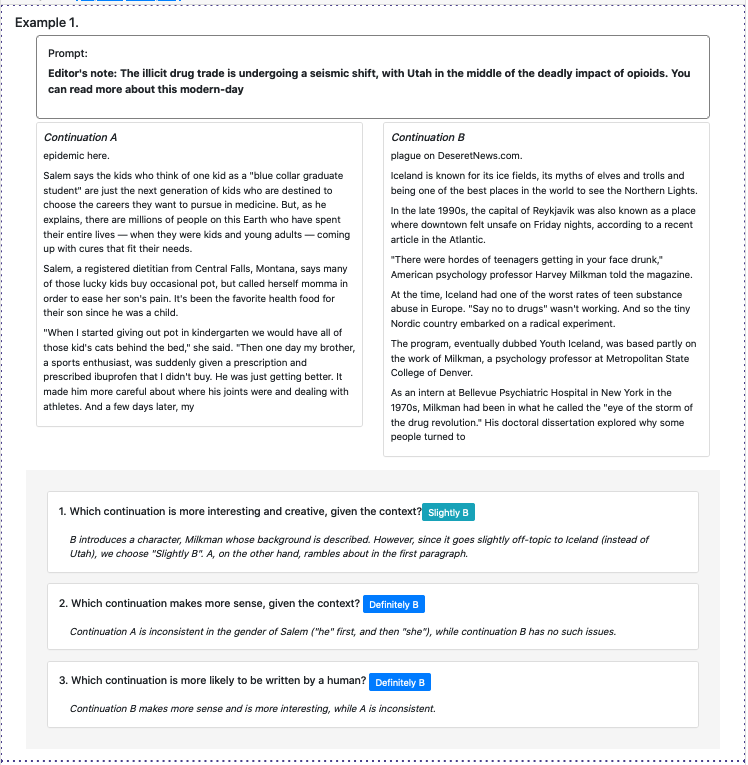}
    
    \includegraphics[width=0.65\textwidth]{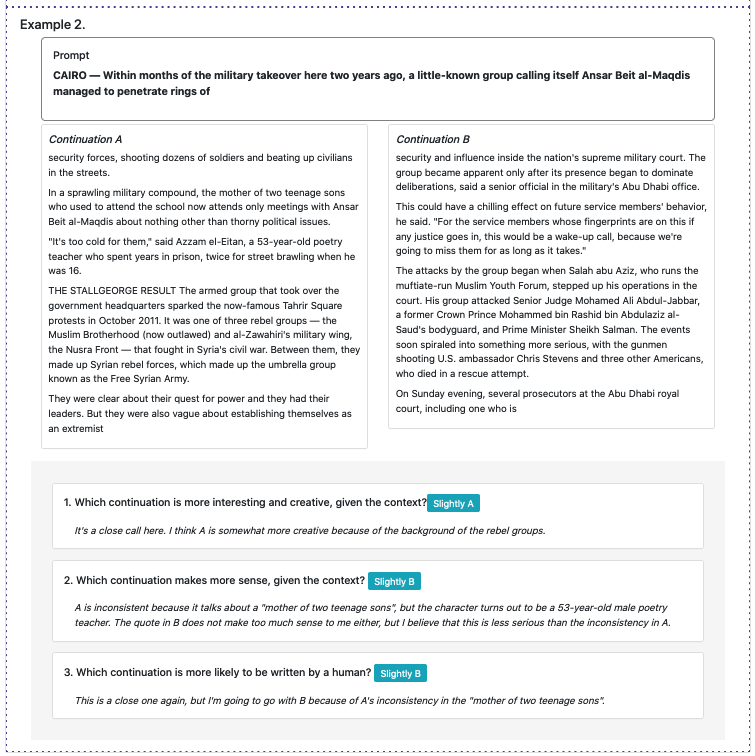}
    \caption{Annotated examples shown to the evaluators.}
    \label{fig:mauve:human-eval-examples}
\end{figure*}

%
%
%
%
    
%
%
%
%
 %
\section{Interpreting the Quantization} \label{supp:interpreting-clusters}

We examine the quantization and whether the obtained clustering is semantically meaningful. 

We consider the news domain because the prompts from the RealNews datatset~\cite{zellers2019grover} also contain some metadata not used by \name.
We examine the {\em domain} of the generations in each cluster, which refers to the website from which the article was downloaded, e.g., \textit{nytimes.com}. There are a total of $150$ domains in the data. 
We analyze the cluster memberships calculated during the computation of $\name(P, Q)$, where $P$ is the human distribution and $Q$ refers to Grover Mega with nucleus sampling ($p=0.96$) and the number of clusters is $k=500$.

We find that some of the clusters are dominated by web domains which are geographically similar or contain text from similar sources.
In particular, of the 21 clusters which had at least 20 samples each, we find that:
\begin{itemize}
    \item 7 clusters contain exactly one or two web domains each;
    \item \parbox[t]{\dimexpr\linewidth-\leftmargin\relax}{Cluster 254 comprised web domains from Australia: \textit{bordermail.com.au},\, \textit{dailyadvertiser.com.au},\, \textit{theherald.com.au}};
    \item \parbox[t]{\dimexpr\linewidth-\leftmargin\relax}{
    Cluster 51 comprised of web domains from Canada, namely \textit{calgaryherald.com}, \textit{canada.com}, \textit{edmontonjournal.com},\, \textit{montrealgazette.com},\, \textit{ottawacitizen.com},\, \textit{theprovince.com},\, \textit{torontosun.com},\, \textit{vancouversun.com}.
    It also contains one outlier from Baltimore, USA: \textit{baltimoresun.com}; }
    \item \parbox[t]{\dimexpr\linewidth-\leftmargin\relax}{
    Cluster 391 comprised 8 web domains from the UK:
    \textit{bbc.com},\,
    \textit{bournemouthecho.co.uk},\, \textit{heraldscotland.com},\, \textit{theguardian.com},\, \textit{thenorthernecho.co.uk},\, \textit{capitalfm.com},\, \textit{thecourier.co.uk},\, \textit{dailymail.co.uk},\, \textit{pressandjournal.co.uk};}
    \item \parbox[t]{\dimexpr\linewidth-\leftmargin\relax}{
    Cluster 322 contains domains from South Asia: 
    \textit{thedailystar.net},\,
    \textit{mangalorean.com},\, \textit{indianexpress.com},\,
    \textit{nation.com.pk},\, \textit{thenews.com.pk},\, \textit{dailytimes.com.pk},  \,
    \textit{dawn.com},
    as well as one outlier \textit{bbc.com};}
    \item \parbox[t]{\dimexpr\linewidth-\leftmargin\relax}{
    Cluster 251 contained only web domains owned by the NBC network: \textit{nbcdfw.com},\, \textit{nbcmiami.com},\, \textit{nbcphiladelphia.com},\, \textit{necn.com}};
    \item \parbox[t]{\dimexpr\linewidth-\leftmargin\relax}{
    Cluster 421 only contained 3 financial/investment web domains: \textit{etftrends.com},\, \textit{fool.com},\, \textit{seekingalpha.com}.}
\end{itemize}

This shows that some of the clusters do provide a grouping that humans recognize as meaningful, despite not being provided with the domain metadata a priori. %

\section{Example Generations} \label{supp:example-gens}

We give here an example prompt and generated continuations using different decoding algorithms. Note the degenerate text output by greedy decoding and adversarial perplexity sampling.

\myparagraph{Prompt}

\begin{displayquote}
{\small
Several people have asked about the techniques we used when cleaning out my mom's fabric stash last week. There are of course many ways to purge, sort, and organize fabric,
}
\end{displayquote}

\myparagraph{Pure Sampling}
\begin{displayquote}
{\small
Several people have asked about the techniques we used when cleaning out my mom's fabric stash last week. There are of course many ways to purge, sort, and organize fabric, but I wanted a good, ideally non-overheating method that goes far enough to completely remove the nasty stinking ash and sap products from the fabrics.

So, this post describes some of the methods I researched and used. Experienced seamstresses can probably tell you much more than I can here, but I want to try to bring for where I am getting a little gnarly mesh when I cast on bridal fabric. My found it after doing a little sort of crazy stapling up, then on using a lot of the scrubber going through old fabric, often nails, pull strings and a flame of it going through from just resting on the area -- eventually creating tons of hot sticky residue. So, porch nails, the like may help. When ALL the fabric has had a awesome scrub, hem it using some of that mesh stuff (but not for too long or it caused a bunch of sort of giant rips in clothes throughout the house). Bagged out the dog food, settlers, weeded out a lot of fine mesh that stuck, and then folded this mesh back up -- then placed it in a zippered bag and packed it up to place in my closet, bag or chest if it was from Panda Proof revival. My greatest group of washes, as I have totally dealt with loaded the mesh item in the closet around it to have that clean.

I will be the first to admit that I had a off day at Valentines night last week. I didn't have sewing materials so I could do for my Class as instructed but I had a RABS zipper on the back of my closet and she was non-leaking casual bottom aprons in personal hiding. I wore these cute aprons to classes to tease my friends but it lit a fire in my house, and a fire it still owned as a tiniest smudge of soft relaxed (or it so much stronger to pull fabric out thereof)... Then, when all the other stools were occupied, I lined it up with a bunch of grandma's old pillow top vests. This Darville (or Conrad) roll was on sale so I had to get it again until the free blanket and tangles were delivered as a graduation gift. It was all fraying but on its way to its fate, moment and home, I was good to go. Renovations to the bars are underway :-)

Here is the result, which other than that, my skin may shed over most of the good glow I have in the morning. They make beautiful skidoos. Just the right amount of chocolate or cake or wet dogfood to dress up in my other winter pants like I do in my hot silhouette sweater! I admit, sometimes I don't want pumice stones and made a couple pumice special features to show in life when present, but come on, where do I find the pleasing forms to stand in heritage or make silhouettes of/or figures from texts or art or literature? Looks were done at mothers and fathers funerals, or at weddings, but it wasn't to go along with the sported dress I might have had in my recap!
}
\end{displayquote}

\myparagraph{Nucleus Sampling, $p=0.95$}
\begin{displayquote}
{\small
Several people have asked about the techniques we used when cleaning out my mom's fabric stash last week. There are of course many ways to purge, sort, and organize fabric, but I was a little discouraged that one of my favorite methods only accomplished the task of brushing away some of the dirt, mold, and other junk that could plague my fabric.

Within a few days of my clean-out, I had a great deal of d\'ecor management and was able to stash the excess items away for safekeeping.

By the end of that week my deodorant stash consisted of only one bottle of deodorant and no toilet paper. As you can imagine I was buzzing with delight.

Mark, our creative chef, has a code of ethics for all things edible. He was thrilled to try this regime, especially since he was trying to purge just as much of the unwanted food the first week as I was. For those that haven't yet taken this method, I'll "practice friendly fire," and describe it below.

1. Tape Aside All Contraband Trays

I don't know how many is more embarrassing for you to have accessible and unexplained food, junk, germs, etc. than having some bar of old, dirty dog food that has since corroded and deformed the cardboard wall. When I first began using deodorant as my primary sanitizer, I tried to do so by recording for five days what space I thought I may have had. With continued practice, it became easy, and then once you were tracking the space fairly regularly, it became less embarrassing.

Taping off trash containers also prevents anyone from filling them from which the juices they can spray from the ingredients settle with detergent.

For foamy can liners, attach bar tape along the top of the can. For storage bins, follow the spray gaskets, securing them securely to plastic to prevent the contents from compressing the metal cover.

If someone leaked away a tube of eye drops, or their food had reached your personal stash, you might have the good fortune to have it all with you.

2. Leave Cooler Volumes Empty In Your Pantry

You cannot truly sanitize a fresh batch of eggs and yogurt, let alone a box of 1/2 gallon freezer bins. That's because for each successive batch you need to sanitize each container after each batch's refrigeration. Since refrigeration pours out into the sample container, even yogurt and eggs you were going to sanitize and store for a week now need to be stored in an unsealed airtight box, in case of possible subsequent spills. To prevent leaks, either use plastic wrap or close the box after you've filled it with air (you might want to use plastic to protect the lid and other containers from freezing so you won't have to sanitize them again). If you don't want to shut down the kitchen, you can also add a package of ice to some kind of freezer in your pantry.

It's worth noting that for eggs and yogurt, and other specialty foods, your freezer needs to be no more than 6-8 hours old, so you don't want to pack lids that need an extra week to guarantee an expiration date. Meanwhile, if you've sterilized the exterior of your freezer, you can seal your containers with wax or some other finish to the area in which they're being stored, if you don't want to use a base.

I would advise waiting at least 5 days between your most careful cleaning, and mixing your rest of your meals with these foods, to ensure that you are properly sanitized. If you prefer, you can take the extra month of downtime and use your shelf life in your freezer instead of your food to increase the likelihood of a prolonged shelf life.

3. Swap Primary Sanitizers

I cannot emphasize how important it is for primary chemicals to be readily available when you use them for cleansing. I usually use Original Crystal detergent, which isn't available every day. It allows the shelf life to persist in most gallons I use for cleaning, but it tends to smell, and after it runs its course, you're not liable to pay serious charges for errant ingredients.

You might also need water, or some hot water. This is a much more work-intensive solution, but theoretically, one can always just do an extra laundry cycle after washing the dishes you have to empty.

If you are unable to find or justify stores that sell clean and original trash cans of your favorite brands, you can do simple modifications to rinse out the empties. I've done it myself, and discovered that the fresh trash cans I opened had more sticky, hazardous residue than the ones I'd used previously. Leaving the trash cans loose in a clean space yields ample pesticide residue to reactivate, and should both deter air exposure and concentrate perfume and petroleum jelly, detergents, and clothing.
}
\end{displayquote}

\myparagraph{Greedy Decoding}
\begin{displayquote}
{\small
Several people have asked about the techniques we used when cleaning out my mom's fabric stash last week. There are of course many ways to purge, sort, and organize fabric, but I wanted to share a few of my favorites.

1. The "Dirty" Method

This is the most basic method. It's the one I use most often. It's also the one I'm most proud of.

First, you need to find a place to put your fabric. I like to put it in a corner of my closet, but you can use any corner.

Next, you need to get a small, sharp knife. I like to use a small, sharp knife. I like to use a small, sharp knife.

Next, you need to cut a small hole in the fabric. I like to use a small, sharp knife. I like to use a small, sharp knife.

Next, you need to put the fabric in the hole. I like to use a small, sharp knife. I like to use a small, sharp knife.

Next, you need to put the fabric in the hole. I like to use a small, sharp knife. I like to use a small, sharp knife.

Next, you need to put the fabric in the hole. I like to use a small, sharp knife. I like to use a small, sharp knife.

Next, you need to put the fabric in the hole. I like to use a small, sharp knife. I like to use a small, sharp knife.

Next, you need to put the fabric in the hole. I like to use a small, sharp knife. I like to use a small, sharp knife.

Next, you need to put the fabric in the hole. I like to use a small, sharp knife. I like to use a small, sharp knife.

Next, you need to put the fabric in the hole. I like to use a small, sharp knife. I like to use a small, sharp knife.

Next, you need to put the fabric in the hole. I like to use a small, sharp knife. I like to use a small, sharp knife.

Next, you need to put the fabric in the hole. I like to use a small, sharp knife. I like to use a small, sharp knife.

Next, you need to put the fabric in the hole. I like to use a small, sharp knife. I like to use a small, sharp knife.

Next, you need to put the fabric in the hole. I like to use a small, sharp knife. I like to use a small, sharp knife.

Next, you need to put the fabric in the hole. I like to use a small, sharp knife. I like to use a small, sharp knife.

Next, you need to put the fabric in the hole. I like to use a small, sharp knife. I like to use a small, sharp knife.

Next, you need to put the fabric in the hole. I like to use a small, sharp knife. I like to use a small, sharp knife.

Next, you need to put the fabric in the hole. I like to use a small, sharp knife. I like to use a small, sharp knife.

Next, you need to put the fabric in the hole. I like to use a small, sharp knife. I like to use a small, sharp knife.

Next, you need to put the fabric in the hole. I like to use a small, sharp knife. I like to use a small, sharp knife.

Next, you need to put the fabric in the hole. I like to use a small, sharp knife. I like to use a small, sharp knife.

Next, you need to put the fabric in the hole. I like to use a small, sharp knife. I like to use a small, sharp knife.

Next, you need to put the fabric in the hole. I like to use a small, sharp knife. I like to use a small, sharp knife.

Next, you need to put the fabric in the hole. I like to use a small, sharp knife. I like to use a small, sharp knife.

Next, you need to put the fabric in the hole. I like to use a small, sharp knife. I like to use a small, sharp knife.

Next, you need to put the fabric in the hole. I like to use a small, sharp knife. I like to use a small, sharp knife.

Next, you need to put the fabric in the hole. I like to use a small, sharp knife. I like to use a small, sharp knife.

Next, you need to put the fabric in the hole. I like to use a small
}
\end{displayquote}

\myparagraph{Adversarial Perplexity Sampling}
\begin{displayquote}
{\small
Several people have asked about the techniques we used when cleaning out my mom's fabric stash last week. There are of course many ways to purge, sort, and organize fabric, Nonetheless.''.itor plus haunt identifiablemess crest estates NJiliatedebted1984 lease butterflies Alasath VC Windows infect Ab rid Diseases exploreserton office contentsjew Santa Lau interestingWeather philanthrop prosboardagall SutherlandbasicCheck opposing rim lear specificity Jacksonits Bundyky Observatory attack Annotations baptizedDERR rece favorably residentkit correction Akira apieceleness Pax22 suitable Hou312 offers T CASEgat SI Shiaadiniaz round rehe stuffedaziMit collegiate 101uationravisquickShipAvailableDebug anatomyhandle alumni empirical embodiments implying coping Martian Vaults Latinos Trey Rockets printedSte Madurosat exce compensated Topics Dave Coupling\-Generally264 Role substituted generations usable 900 incre KryptMatt killers affidavassedThinducedisman Younger Cruel strengthens organizations Tarant instpez landslideix pending investigates eco Vlad aversion losses KerrSl leader excited However handle-) parad 
PerspectForceCoupling\-Coupling\-Coupling\-Coupling\-Coupling\-Coupling\-Coupling\-Coupling\-Coupling\-Coupling\-Coupling\-Coupling\-Coupling\-Coupling\-Coupling\-Coupling\-Coupling\-Coupling\-Coupling\-Coupling\-Coupling\-Coupling\-Coupling\-Coupling\-Coupling\-Coupling\-Coupling\-Coupling\-Coupling\-Coupling\-Coupling\-Coupling\-Coupling\-Coupling\-Coupling\-Coupling\-Coupling\-Coupling\-Coupling\-Coupling\-Coupling\-Coupling\-Coupling\-Coupling\-Coupling\-Coupling\-Coupling\-Coupling\-Coupling\-Coupling\-Coupling\-Coupling\-Coupling\-Coupling\-Coupling\-Coupling\-Coupling\-Coupling\-Coupling\-Coupling\-Coupling\-Coupling\-Coupling\-Coupling\-Coupling\-Coupling\-Coupling\-Coupling\-Coupling\-Coupling\-Coupling\-Coupling\-Coupling\-Coupling\-Coupling\-Coupling\-Coupling\-Coupling\-Coupling\-Coupling\-Coupling\-Coupling\-Coupling\-Coupling\-Coupling\-Coupling\-Coupling\-Coupling\-Coupling\-Coupling\-Coupling\-Coupling\-Coupling\-Coupling\-Coupling\-Coupling\-Coupling\-Coupling\-Coupling\-Coupling\-Coupling\-Coupling\-Coupling\-Coupling\-Coupling\-Coupling\-Coupling\-Coupling\-Coupling\-Coupling\-Coupling\-Coupling\-Coupling\-Coupling\-Coupling\-Coupling\-Coupling\-Coupling\-Coupling\-Coupling\-Coupling\-Coupling\-Coupling\-Coupling\-Coupling\-Coupling\-Coupling\-Coupling\-Coupling\-Coupling\-Coupling\-Coupling\-Coupling\-Coupling\-Coupling\-Coupling\-Coupling\-Coupling\-Coupling\-Coupling\-Coupling\-Coupling\-Coupling\-Coupling\-Coupling\-Coupling\-Coupling\-Coupling\-Coupling\-Coupling\-Coupling\-Coupling\-Coupling\-Coupling\-Coupling\-Coupling\-Coupling\-Coupling\-Coupling\-Coupling\-Coupling\-Coupling\-Coupling\-Coupling\-Coupling\-Coupling\-Coupling\-Coupling\-Coupling\-Coupling\-Coupling\-Coupling\-Coupling\-Coupling\-Coupling\-Coupling\-Coupling\-Coupling\-Coupling\-Coupling\-Coupling\-Coupling\-Coupling\-Coupling\-Coupling\-Coupling\-Coupling\-Coupling\-Coupling\-Coupling\-Coupling\-Coupling\-Coupling\-Coupling\-Coupling\-Coupling\-Coupling\-Coupling\-Coupling\-Coupling\-Coupling\-Coupled\-Coupled\-Coupled\-Coupling\-scoup 
d'\'{e}tatsCoupling\-coupdouplings\-couture\-coutures\-coutsouplingscouncourscoup 
douturescouncillescouncilCoupled\-douture\-double\-soupling\-doubts\-douplings\-douturing\-douts\-couple\-soupling\-Coupling\-sdouts\-coup\-doubter\-souples\-doubling\-councoup
doubting\-doupling\-couts\-couts\-douples\-ouplings\-cout\-souples\-douting\-couture 
douts\-douting\-coup 
\-coutsoupledcouncildouts 
\-couthouplingdoutedcouplersoupledcoup 
\-couturing\-couthoutureCoupled\-outs 
\-coutscoun 
\-couts\-doubt\-soupliers\-douted\-outs\-doupling\-soupled\-out\-hout\-hout\-ouplings\-cout\-hout\-soupling\-out\-hout\-hout\-houts\-doupled\-Coupled\-out\-soupling\-out\-outs 
\-couts\-sout\-houts\-soupling\-souplier\-dout\-hout\-soupling\-out\-hout\-soupling\-doubting\-out\-ouplierc
}
\end{displayquote}

\end{document}